\documentclass[letterpaper]{article} 
\usepackage{aaai23}  
\usepackage{times}  
\usepackage{helvet}  
\usepackage{courier}  
\usepackage[hyphens]{url}  
\usepackage{graphicx} 
\usepackage{natbib}  
\usepackage{caption} 
\usepackage{algorithm}
\usepackage{algorithmic}

\usepackage{newfloat}
\usepackage{listings}
\DeclareCaptionStyle{ruled}{labelfont=normalfont,labelsep=colon,strut=off} 
\floatstyle{ruled}
\newfloat{listing}{tb}{lst}{}
\floatname{listing}{Listing}
\usepackage[utf8]{inputenc} 
\usepackage[T1]{fontenc}    
\usepackage{url}            
\usepackage{booktabs}       
\usepackage{amsfonts}       
\usepackage{nicefrac}       
\usepackage{microtype}      
\usepackage{xcolor}         

\usepackage{algorithm}
\usepackage{algorithmic}
\usepackage{multirow}

\usepackage{amsmath}
\usepackage{mathtools}
\usepackage{amssymb}
\usepackage{comment}
\usepackage{enumitem}
\usepackage[subtle]{savetrees}

\DeclareMathOperator{\EX}{\mathbb{E}}

\newtheorem{theorem}{Theorem}
\newtheorem{lemma}{Lemma}

\newtheorem{definition}{Definition}

\DeclareMathOperator*{\argmax}{arg\,max}

\newenvironment{proof}{\noindent\emph{\textbf{Proof. }}\ignorespaces}%
{\hspace*{\fill}$\Box$\par}

\usepackage{graphicx}
\usepackage{subfigure}

\usepackage{newfloat}
\usepackage{listings}
\DeclareCaptionStyle{ruled}{labelfont=normalfont,labelsep=colon,strut=off} 
\floatstyle{ruled}
\newfloat{listing}{tb}{lst}{}
\floatname{listing}{Listing}
\title{Towards Efficient and Domain-Agnostic Evasion Attack \\ with High-dimensional Categorical Inputs}
\author{Hongyan Bao\textsuperscript{\rm 1},
    Yufei Han\textsuperscript{\rm 2},
    Yujun Zhou\textsuperscript{\rm 1},
    Xin Gao\textsuperscript{\rm 1},
    Xiangliang Zhang\textsuperscript{\rm 1,\rm 3,}\thanks{Corresponding author.} 
}
\affiliations{
    \textsuperscript{\rm 1}King Abdullah University of Science and Technology ~~
    \textsuperscript{\rm 2}INRIA ~~
    \textsuperscript{\rm 3}University of Notre Dame

}

\usepackage{bibentry}

\begin{document}

\maketitle
\begin{abstract}
Our work targets at searching feasible adversarial perturbation to attack a classifier with  high-dimensional categorical inputs in a domain-agnostic setting.
This is intrinsically a NP-hard knapsack problem where the exploration space becomes explosively larger as the feature dimension increases. Without the help of domain knowledge, solving this problem via heuristic method, such as Branch-and-Bound, suffers from exponential complexity, yet can bring arbitrarily bad attack results. We address the challenge via the lens of multi-armed bandit based combinatorial search. Our proposed method, namely FEAT, treats modifying each categorical feature as pulling an arm in multi-armed bandit programming. Our objective is to achieve highly efficient and effective attack using an Orthogonal Matching Pursuit (OMP)-enhanced Upper Confidence Bound (UCB) exploration strategy. Our theoretical analysis bounding the regret gap of FEAT guarantees its practical attack performance. In empirical analysis, we compare FEAT with other state-of-the-art domain-agnostic attack methods over various real-world categorical data sets of different applications. Substantial experimental observations confirm the expected efficiency and attack effectiveness of FEAT applied in different application scenarios. Our work further hints the applicability of FEAT for assessing the adversarial vulnerability of classification systems with high-dimensional categorical inputs. 
\end{abstract}

\section{Introduction}\label{sec:introduction}

Adversarial evasion attacks have been witnessed in many real-world data analytical applications\cite{goodfellow2014explaining,cartella2021adversarial,suciu2019exploring,stringhini2010detecting,imam2019survey}, including text processing \cite{yang2020greedy,papernot2016crafting} and image recognition {\cite{goodfellow2014explaining,szegedy2013intriguing}}. 
Despite the flourish efforts on evasion attacks with continuous inputs, such as image and video contents \cite{goodfellow2014explaining,szegedy2013intriguing,CarliniSP2018,biggio2012poisoning}, 
much less attention has been paid to explore the adversarial threat against on machine learning systems with categorical inputs. Categorical data exist prevalently in real-world trust-critical applications, like cyber attack detection \cite{shu2020fakenewsnet,wang2017liar,deepcase} and medical diagnosis. For instance, detecting cyber attacks usually depends on categorical behavioral signatures of the target IT infrastructures, including malware execution traces, malicious network communication logs, and system event logs \cite{deepcase,pendlebury2019tesseract}. Machine Learning-based medical diagnosis is often conducted by combining and encoding qualitative results of various medical tests. 
Unlike continuous measurements such as pixel intensities, each categorical feature is valued with mutually exclusively category values. These optional category values have no intrinsic ordering. Conducting adversarial perturbations on categorical features is therefore in nature an \textit{NP-hard knapsack problem} \cite{wang2020attackability}. 
\textbf{On one hand}, popular gradient-guided evasion attack methods against deep learning models \cite{Goodfellow2015} become infeasible as computing gradients directly over categorical variables is not applicable. \textbf{On the other hand}, classic heuristic search solutions, e.g., Branch-and-Bound and trial-and-error methods, suffer from high complexity and lack guaranteed quality of the derived attack results, which can lead to arbitrarily bad attack performances.
It is therefore difficult to define a computationally efficient strategy to produce effective adversarial perturbations over categorical inputs.

The current study in solving the adversarial attack problem over categorical inputs falls into two groups.
\textbf{First}, \textbf{domain-specific} knowledge is applied to narrow down the combinatorial perturbation space, and used as constraints to preserve semantic/function integrity of the perturbed instances \cite{li2020bert-attack, zang2020word,gao2018black,li2018textbugger,jin2020bert,samanta2017towards,papernot2016crafting,ma2018risk,wang2020attackability,suciu2019exploring,Narodytska2017cvprw,Croce2019iccv,fabio2020sp}. 
Such domain-specific dependency limits the adaptive potential of the attack method across different applications.     
Moreover, domain-specific knowledge may not be always readily available. For example, the threat settings of cyber attacks vary drastically across different attack techniques and IT system architectures \cite{deepcase}. Encoding domain-specific contexts of various intrusion incidents require expensive investigation overheads on a case-by-case basis. Besides, 
system threats may stay unknown to security analysts when an attack is delivered. It is impossible to define domain-specific rules for the zero-day attack events. 
The absence of \emph{a principled and domain-agnostic adversarial attack protocol} makes it difficult to provide an attack-as-a-service pipeline to evaluate the adversarial vulnerability of different trust-critical applications. \textbf{Second}, forward stepwise greedy search (FSGS) has been adopted in \cite{ebrahimi2018hotflip,wang2020attackability} as a domain-agnostic method to generate feasible adversarial modifications to categorical data. Domain-specific constraints over the feasible adversarial modifications can be used as a plug-in to FSGS. However, the greedy search method induces prohibitively expensive computational cost as the number of categorical features and/or the optional category values in the target input become large. The intense overheads prevent the adversary from organizing efficient attacks and/or makes it inapplicable to assess the adversarial vulnerability of a target machine learning system in practices. 

To address the limits of current study, we propose an orthogonal matching pursuit (OMP)-boosted multi-armed bandit search to deliver a \textbf{f}ast and \textbf{e}ffective \textbf{a}dversarial a\textbf{t}tack in a high-dimensional combinatorial search space, named as FEAT hereafter. FEAT adopts orthogonal matching pursuit \cite{elenberg2018restricted,wang2020attackability} to identify the most sensitive categorical features to perturb in each round of the iterative attack process. Over the selected candidate features, FEAT considers modifying each categorical feature as triggering an arm in a multi-armed bandit game. Exploring the feasible combinations of categorical feature modifications can thus be guided with Upper Confidence Bound (UCB)-driven exploration in a computationally efficient way. 
The advantages of FEAT are summarized as follows. 

\begin{itemize}[leftmargin=*]
    \item \textbf{Computationally-economic attack with high-dimensional categorical inputs.} The computational complexity of FEAT is linear to the number of modified features. In contrast, the complexity of the state-of-the-art domain-agnostic attack methods proposed by \cite{QiSysML2018} and \cite{wang2020attackability} grow as a geometric series of the number of modified features. Empirically FEAT costs {1/10} -- {1/3} of the overheads compared to the state-of-the-art domain-agnostic and domain-specific attack methods, while requiring less features to modify to deliver highly successful attacks.  
    \item \textbf{Theoretical guaranteed attack performance.} We set up an upper bound of the expected regret of FEAT in our analysis. It applies to a general Lipschitz-smooth deep learning-based target classifier with categorical inputs, which guarantees the attack performance of FEAT in general attack scenarios. 
    \item \textbf{Domain-agnostic adaption to various different applications.} We evaluate FEAT over 4 categorical data sets collected from various real-world applications. The empirical observations confirm FEAT is well adapted to different application domains and show its consistently superior attack effectiveness and efficiency, comparing to the state-of-the-art domain-agnostic and domain specific attack baseline methods. The results also reconcile with the theoretical guarantee to the success of FEAT attacking general classifiers. 
\end{itemize}

\section{Related Work}
 \cite{wang2020attackability,QiSysML2018,yang2020greedy,ebrahimi2018hotflip} proposed to adopt forward stepwise greedy search (\textit{FSGS}) based methods in generating discrete adversarial samples in a domain-agnostic way. \textit{FSGS} is an iterative process. In each iteration, it considers all possible combination of each candidate categorical feature with the subsets of the adversarially modified features in previous rounds. \textit{FSGS} then chooses the candidate feature that can achieve the largest marginal gain of the attack objective. Though \textit{FSGS} plays as a domain-agnostic attack method, it can also use additional domain-specific constraints to reduce the size of feasible feature modifications. However, the bottleneck of \textit{FSGS} is that its computational cost grows as a geometric series of the number of the modified features. It becomes prohibitively expensive as the dimension of the target discrete instance is high. 

Domain-specific adversarial attack mostly target at text classifiers \cite{papernot2016crafting,Miyato2016AdversarialTM,samanta2017towards,Yang2018GreedyAA,gao2018black,li2018textbugger,jia2017adversarial,jin2020bert}.
\cite{gao2018black} developed scoring functions to evaluate the importance of each word in a sentence and proposed to modify the top-ranked words identified by the scoring functions. Similarly, \cite{papernot2016crafting} selected the word to replace where the variation of the word's embedding vector is best aligned to the gradient direction of the target model. In contrast, \cite{jia2017adversarial} proposed to insert distraction sentences into a target text sample with a human-involved loop to fool a reading comprehension system. \cite{samanta2017towards} added linguistic constraints over the pool of candidate-replacing words. Recently, \textit{TextBugger} \cite{li2018textbugger} used typo-based perturbation for each word to get the candidates of feasible modifications over each word. \textit{TextFooler} \cite{jin2020bert} used the similarity of word embedding to select the candidates of each words to attack. These methods depend on semantic/syntactic rules to shrink the feasible set of text modifications. Besides, they adopt trial-and-error search to explore possible text modifications. They lack the guarantee to the solution quality to the knapsack based discrete evasion attack problem. Their attack performances, i.e., the success of attack complying to the attack budget constraint, may thus vary drastically over different target inputs. 

\section{Preliminaries}
Let $\mathbf{x} = \{x_1,x_2,x_3,...,x_N\}$ denote a discrete input instance with $N$ categorical features. 
Each $x_i$ may take any of $M$ ($M\geq{1}$) categorical values. 
We cast each optional category value of a discrete feature $x_{i}$ to a $D$-dimensional pre-trained embedding vector, e.g., $\mathbf{e}^{j}_{i} \in{R^{D}},\, j=1,2,...,M$. We introduce binary indicators $\mathbf{b} = \{b^{j}_{i}\}$, $i$=$1,2,...,N$, $j$=$1,2,...,M$,
where $b^{j}_{i} = 1$ when the $i$-th categorical feature $x_{i}$ takes the $j$-th categorical value of $x_i$,  
and $b^{j}_{i} = 0$ otherwise. 
One instance $\mathbf{x}$ can then be represented by stacking the embedding vectors of each categorical variable $x_{i}$ as an $R^{N*M*D}$ tensor with $\mathbf{x}_{\{i,j,:\}} = b^{j}_i{\mathbf{e}^{j}_i}$. We summarize all the notations used for describing the design of the proposed method in Appendix.A. 

With this setting, the adversarial perturbation over $\mathbf{x}$ is to modify $\mathbf{{b}}$  to $\mathbf{\hat{b}}$. 
$\hat{b}^{j}_{i}={b}^{j}_{i}$ denotes $x_{i}$ is not perturbed. Otherwise, $\hat{b}^{j}_{i}\neq{b}^{j}_{i}$ indicates the corresponding feature $x_{i}$ is changed. Depending on the type of attacks, i.e., \emph{insertion}, \emph{deletion} or \emph{substitution}, $\hat{b}^{j}_{i}$ can be valued in different ways. \emph{Insertion} is to let $\hat{b}^{j}_{i}=1$, given $b^{j}_{i}=0,\,\forall j=1,...,m$.
\emph{Deletion} is to let $\hat{b}^{j}_{i}=0$, given $b^{j}_{i}=1$. 
\emph{Substitution} is to let $\hat{b}^{j}_{i}=1,  \hat{b}^{j'}_{i}=0$, given $b^{j}_{i}=0$, $b^{j'}_{i}=1, j\neq j'$.
A modified instance $\hat{\mathbf{x}}$ can thus be written as $\mathbf{\hat{x}}_{\{i,j,:\}} = \hat{b}^{j}_i{\textbf{e}^{j}_i}$. The classifier $f$ outputs decision confidence $f_{y_k}$ ($k=1,2,3,...,{K}$) with respect to different class labels. Without loss of generality, let ${y}_K$ denote the true class label of ${x}$ and all the other ${y}_{k}$ ($k=\{1,...,{K}-1\}$) are the potential targets of an evasion attack. Given an input $\mathbf{x}$, the goal of evasion attack is to increase the misclassification risk of $f$ over $\mathbf{x}$, i.e., making $f_{y_K}(\mathbf{x},{\hat{\mathbf{b}}})$ as low as possible and $f_{y_k}(\mathbf{x},{\hat{\mathbf{b}}})$ of any of the $k$ except ${K}$ (\emph{non-targeted attack}) as high as possible simultaneously. The combinatorial optimization problem of evasion attack is defined below:
\begin{definition}
$f:\mathbf{x}\rightarrow{y}$ denotes a classifier with categorical inputs $\mathbf{x}$. The adversary aims to maximize the misclassification confidence $f$ complying the constraint of the attack budget $\varepsilon$, i.e. the maximum number of modified categorical variables in $\mathbf{x}$.
\begin{equation}\label{eq:adv_risk}
\footnotesize
\hat{\mathbf{b}}^* = \underset{\hat{\mathbf{b}},|\text{diff}(\mathbf{b},\hat{\mathbf{b}})|\leq{\varepsilon}}\argmax f_{y_{k}}(\hat{\mathbf{x}}_{\{i,j,:\}}=\hat{b}^{j}_{i}\mathbf{e}^{j}_{i}),\quad  y_{k}\neq{y_{K}}
\end{equation}
where $\mathbf{x}_{\{i,j,:\}} = b^{j}_{i}\mathbf{e}^{j}_{i}$ and  $\mathbf{\hat{x}}_{\{i,j,:\}} = \hat{b}^{j}_{i}\mathbf{e}^{j}_{i}$ are the unperturbed  and the adversarially tuned instance. 
\end{definition}

\section{The Algorithm Design of FEAT}
Our design of FEAT illustrated in \textbf{Algorithm} \ref{alg:UCBVRG} is inspired by the analogy between Multi-Armed Bandit (MAB)-based combinatorial search and the attack problem given in Definition.1. Finding \emph{one categorical feature $x_{i}$ in the input instance $\mathbf{x}$ to perturb} is analogous to selecting \emph{one arm to pull in an MAB game}. Each arm is characterized by the distribution of the received rewards. Similarly, taking an action to modify the category value of one discrete feature can also cause the variation of the decision output of the target classifier $f$ as a feedback. 

More specifically, FEAT defines an iterative MAB search in the discrete feature space to solve the knapsack optimization problem in Eq.\ref{eq:adv_risk} (see \textbf{Algorithm} \ref{alg:UCBVRG} \textbf{Line 6-15}). Given a categorical feature $x_{l}$ of $\mathbf{x}$, $t_l$ denotes the number of times when $x_{l}$ is selected to perturb after $t$ iterations of the MAB-driven search.{Inheriting the terms used in Eq.\ref{eq:adv_risk}, the reward of modifying each $x_l$ in current $t^{c}$ iteration (noted as $G_{l,t^{c}}$ in Eq\eqref{eq:reward}) is defined as the maximum gap $m_{f}$ between the classifier's output over any wrong label $k$ and the correct label $K$ by modifying $x_{l}$. 
\begin{equation}\label{eq:reward}
\small
    G_{l,t^c} = \max f_{y_{k}}(\hat{\mathbf{x}}_{l,t^{c}}) - f_{y_{K}}({\mathbf{x}}) + \Lambda 
\end{equation}
where $\hat{\mathbf{x}}_{l,t^{c}}$ denotes the adversarially perturbed input instance at the current iteration $t^{c}$ with $x_l$ changed. A constant $\Lambda$ is added to Eq.\ref{eq:reward} to ensure the non-negativeness of the received rewards. In practices, we set $\Lambda=1$.}
In each iteration, the Upper Confidence Bound (UCB) score of each candidate discrete feature can be computed following Eq.\eqref{eq:UCBv}: 
\begin{equation}\label{eq:UCBv}
\small                  
B_{l,t_l,t} = {\Bar{\mu}}_{l,t} + \sqrt{\frac{ \alpha {\Bar{\delta}^2}_{l,t_l}*\log{t} }{t_l}} + \frac{\log{t}}{t_l}
\end{equation}
where ${\Bar{\mu}}_{l,t} \overset{\text{def}}{=} \frac{1}{t}\sum_{t^c=1}^t G_{l,t^c}$ and ${\Bar{\delta}}^2_{l,t} \overset{\text{def}}{=} \frac{1}{t} \sum_{t^c=1}^t (G_{l,t^c}- {\Bar{\mu}}_{l,t})^2$ are the empirical mean and variance of the obtained rewards (increase of the classification confidence produced by $f$) by changing $x_l$ after $t$ iterations of search. \emph{In each iteration, the adversary chooses the candidate feature with the highest UCB score as the target to perturb.} The parameter $\alpha$ is  tunable to make a trade-off between exploration and exploitation of the search for discrete feature perturbations. \textbf{On the one hand}, a larger $\alpha$ extends the exploration covering more new candidate features that have never been tried before. \textbf{On the other hand}, an extremely small $\alpha$ drives the search to lean more towards the highly sensitive features. Modifying any of these features can cause drastic variation of $f$'s decision. 
We traverse different choices of $\alpha$ in FEAT to empirically observe the impact of $\alpha$ over the attack performance.{We leave the parameter sensitivity study in Appendix.E}. According to Theorem.2 in \cite{Audibert2007}, choosing the UCB score as in Eq.\ref{eq:UCBv} ensures that the event of drawing sub-optimal candidate features in the attack process has a decreasingly smaller probability after increasingly more iterations of search. The adversary then conducts the UCB-guided exploration of feasible discrete perturbations to avoid exhaustive search over all the possible combinations of the candidate categorical features.

\begin{algorithm}[ht]
\small
\caption{FEAT: Fast and Effective Adversarial aTtack}\label{alg:UCBVRG} 
\begin{minipage}{0.6\textheight}
\begin{algorithmic}[1]
\small
\REQUIRE
The input $\mathbf{x}$ to perturb, the trained model $f_{y}$, \\
the attack budget $\varepsilon$, the time limit $T_L$, \\ the number of features to select $L$, \\
the number of UCB loops  $\tau$;\,\, 
\ENSURE
the chain of features selected to attack   $S$;\\
\STATE $S_{0}\leftarrow{\emptyset}$,  $\hat{\mathbf{x}} \leftarrow \mathbf{x}$\\
\WHILE{$len(S_t) {\leq} {\varepsilon}$ and \emph{TimeCost}  {$\textless$} ${T_L}$}{
     \STATE $\text{grad}_{i}{\leftarrow}$ The gradient ${\nabla}{f_{y}(\hat{x}_i)}$ for each feature $\hat{x}_i$\\
      \STATE Weight of each feature ${{w_i}} = \text{grad}_{i}/{\sum_{j=1}^N}\text{grad}_{j}$ \\
      \STATE Select top-$L$ features based on ${w_i}$ from $N$ features\\
      \FOR{$l=1,2,\ldots$,$L$,}{
           \STATE $ {\hat{x}_l} \leftarrow $ the $l$-th selected feature \\
           \STATE $G_{l,t^0} =  \max f_{y_{k}}(\hat{\mathbf{x}}_{l,t^0}) - f_{y_{K}}(\mathbf{x}) + \Lambda$ \\
           
          }
          
          \ENDFOR \\
          \FOR{  {$t=1,2,\ldots, \tau$}}{
                 \STATE Update ${\Bar{\mu}}_{l,t-1}$  and ${\Bar{\delta}}^2_{l,t-1}$ in Eq.\eqref{eq:UCBv}
                 \STATE $I_t = {\underset{l\in{\{1,2...,L\}}}{\arg\max}{ B_{l,(t-1)_l,t-1}}} $ \\
                 \STATE ${S_{t}~~{\leftarrow}~~ S_{t-1}~{\cup}~{I_t}}$  
                 \\
                 \STATE Modify  {$I_t$} in $\hat{\mathbf{x}}$    \\
              }
              \ENDFOR
      }
      \ENDWHILE

\end{algorithmic}
\end{minipage}
\end{algorithm}


{The popular heuristic search methods, e.g., the standard UCB and Thompson Sampling (TS), share similar computational complexity according to \cite{agrawal13aistats,auer2002finite}. 
In the high-dimensional feature space, applying these methods directly is prone to fast increasing of the computational cost as the number of the categorical features (noted as $N$) and/or the number of the optional categorical values per feature (noted as $M$) become higher \cite{agrawal13aistats,auer2002finite}. Besides, the regret of TS does not scale polynomially in the feature dimension. TS can perform strictly worse than random choice in the high dimensional case \cite{Zhang2021nips}. The complexity bottleneck of both methods motivates us to adopt a more efficient strategy adapted to the high-dimensional search problem.} Previously, \cite{Wang2008nips} chooses to randomly sample a subset of features to perform the UCB subroutine to reduce the cost. However, blindly sampling subsets of features may miss effective feature perturbations that bring large variation of the classifier's output, eventually hurting the attack performance, as shown in \cite{Wang2008nips}.  

In the proposed FEAT method, we aim to optimise the balance between the exploration coverage and the computational overheads by boosting the UCB-guided search with an orthogonal matching pursuit (OMP)-based feature ranking strategy \cite{buchbinder2014submodular,wang2020attackability}.  Each iteration of the search is composed by first conducting the OMP computing to rank the candidate categorical features according to their influence over the classifier's decision given the current input instance (see \textbf{Algorithm} \ref{alg:UCBVRG} \textbf{Line 3-5}) and then performing $\tau$ rounds of UCB search as the inner iterations over the top $L$ influential features selected by the OMP computation (see \textbf{Algorithm} \ref{alg:UCBVRG} \textbf{Line 6-15}). $\tau$ is a tunable parameter, adjusting the number of search rounds within the selected top $L$ features. 
We choose $\tau$ empirically to 1/3 of the attack budget, which presents consistently good attack success rate with low attack budget cost. 



To perform the OMP-based feature ranking, we relax the binary indicators $b^{j}_{i}$ attached to each categorical feature $x_{i}$ to be continuous and valued within the range $[0,1]$. We then take the gradient of $f(\mathbf{x},\mathbf{b})$ with respect to $\mathbf{b}$, denoted as $\nabla_{\mathbf{b}}(f_{y}(\mathbf{b}))$. The selected candidate categorical features are those with the largest gradient magnitudes $\|\nabla_{\mathbf{b}}(f_{y}(\mathbf{b}))\|$. We reason the rationality of using the OMP-boosted UCB search by establishing the theoretical study first in Theorem.\ref{theorem:bound_gradient}. The theoretical study shows the OMP-based ranking can select highly influential features over the classifier's decision. Furthermore, we build the regret bound of the OMP-boosted UCB search of FEAT in Theorem.\ref{theorem:regert bound}, which further explains the merit of the OMP-ranking in enhancing the efficiency of UCB search. The regret analysis in Theorem.\ref{theorem:regert bound} also provides a theoretical guarantee to the attack performance of FEAT against a general classifier with categorical inputs. We state the two theorems in the followings. Via the theoretical study, we discuss further how the the designed balance between exploration and exploitation is achieved in FEAT.  We provide the corresponding proofs to the two theorems  in Appendix.B and C. 

The reward distribution may drift during the attack process, which poses a challenge of non-stationary rewards to the practices of FEAT. We adapt the UCB-based search to the scenario from two aspects. \textbf{On one hand}, The OMP operation of FEAT restricts the search range to the potentially sensitive features. The rewards of those sensitive features remain relatively more stable than the rest features within a few iterations of search, as shown by Observation.2 in Appendix.F. \textbf{On the other hand}, FEAT re-initialise the UCB-based search every $\tau$ rounds and recompute the OMP-based feature ranking. 
Via this way, the UCB-based search is conducted only within the consecutive $\tau$ inner iterations. Based on the empirical observations in Appendix.F, the reward distribution of the selected $L$ sensitive features can be considered approximately stationary within the $\tau$ inner iterations. Though lack of proof, FEAT provides an empirically feasible environment for the use of the UCB-based search.




\subsection{The Indicator of Feature Importance}

\begin{theorem}\label{theorem:bound_gradient}
\textbf{Gradient as an Indicator.} Let $\mathbf{b}$ indicate the category value assignment of an unperturbed data instance ${\mathbf{x}}$. $\hat{\mathbf{b}}$ and $\hat{\mathbf{b}}'$ indicate the two different sets of the modification over the same unperturbed input $\mathbf{x}$. We further assume $|\text{diff}\,(\mathbf{b},\hat{\mathbf{b}})|\leq{\zeta},|\text{diff}\,(\mathbf{b},\hat{\mathbf{b}}')|\leq{\zeta},|\text{diff}\,(\hat{\mathbf{b}},\hat{\mathbf{b}}')|\leq{\zeta},\,\zeta\geq{1}$. Given a smooth target classifier $f$ with a finite Lipschitz constant, $f_{y_k}(\mathbf{x})$ denotes the decision output of $f$ over any incorrect class label, i.e. $y_k\neq{y_{K}}$ in Definition.1. Let {\footnotesize $\nabla{f}_{y}(\mathbf{x},\hat{\mathbf{b}})_{\nu}$} denote the elements of {\footnotesize $\nabla{f}_{y}(\mathbf{x},\hat{\mathbf{b}})$} corresponding to the difference between $\hat{\mathbf{b}}$ and $\hat{\mathbf{b}}'$, where {\footnotesize $\nu = \textit{diff}(\hat{\mathbf{b}},\hat{\mathbf{b}}')$. }
\begin{equation}\label{eq:gradientbound}
\small
\begin{split}
    |f_{y_k}({\mathbf{x}},\hat{\mathbf{b}'}) - f_{y_k}({\mathbf{x}},\hat{\mathbf{b}})|\leq &\quad \max\{\frac{1}{2m_{k,\zeta}}\|\nabla{f_{y_k}}(\mathbf{x},\hat{\mathbf{b}})_{\nu}\|^2_2,  
   \;\; \\
   &\quad \|\nabla{f_{y_k}}(\mathbf{x},{\hat{\mathbf{b}}})_{\nu}\|_2+M_{k,\Omega_{\zeta}}|\zeta|/2\}\\
\end{split}
\end{equation}
where 
$m_{k,\zeta}$ and $M_{k,\Omega_{\zeta}}$ are the local strong convexity factor and local Lipschitz constant of the target classifier $f$ around the unperturbed input $\mathbf{x}$. 
\end{theorem}

As corroborated by Theorem \ref{theorem:bound_gradient}, the magnitude of each element in $\nabla_{\mathbf{b}}(f_{y}(\mathbf{b}))$ provides a bounded estimator to the marginal contribution of attacking the corresponding categorical feature. The top-ranked candidate features with the highest gradient magnitudes $\|\nabla_{\mathbf{b}}(f_{y}(\mathbf{b}))\|$ are more likely to be the most sensitive features with respect to the adversarial perturbation, compared to those at the tail of the ranking list. We therefore use the OMP-based ranking strategy to narrow down the search regime within the selected top-ranked and potentially sensitive features. Perturbing more sensitive features can produce higher adversarial risk over categorical inputs, according to Theorem.2 in \cite{han2021iclr}. The rationality of integrating this feature ranking step is thus to encourage the UCB-guided search to concentrate more on manipulating the sensitive features, which is more likely to cause larger change to the classifier's output with only a few feature changed. 







\subsection{The Expected Regret Bound of FEAT}

\begin{definition}
For one feature $l$ out of the top-ranked $L$ features, 
the expected regret of perturbing   $l$ is
\begin{equation}\label{eq:regret def}
\bigtriangleup_l  =\mu^* - \mu_{l} \quad (\mu^* = \max_{1\leq l \leq L}{\mu_{l}})
\end{equation}
where $\mu_l$ and $\mu^*$ are the expected and optimal reward received by changing $l$. 
\end{definition}

\begin{theorem}
\label{theorem:regert bound}
\textbf{Perturbing highly sensitive features helps shrink the regret bound of FEAT.} Let $\bigtriangleup_l >0$ and ${\delta^2_{l}}>0$ be the expected regret and the expected variance of the rewards received by modifying each of 
the top-$L$ candidate features. 
The expected regret bound of FEAT after $T$ iterations can be given as:
\begin{equation}\label{eq:regretbound} 
    \EX[{R}_T] \; \leq  \; \sum_{l=1}^L [8(\frac{\delta^2_{l}}{\bigtriangleup_l} + 2)\log T + \frac{\alpha}{\alpha-2} \bigtriangleup_l], 
\end{equation}    
where $\EX[{R}_T] \overset{\text{def}}{=} \sum_{l=1}^L \EX[T_l]  \bigtriangleup_l$. ${T_l}$ is the number of times that feature $l$ is selected after $T$ iterations. 

    
\end{theorem}

Empirically, we observe that the variances $\delta^2_l$ of the received rewards for each top-ranked sensitive features remain low in the attack process. It is possible that these highly sensitive features play important roles in  classification. Any modification over such features (e.g., switching the category value of these features) thus produces consistently large change to the classification output (i.e., obtaining consistently high rewards). As a result, the variance of the reward obtained by perturbing these features is low. The empirical evaluation with respect to the reward variance of these top-ranked sensitive features can be found in the appendix. 

According to Theorem.\ref{theorem:regert bound}, as $\delta^2_l$ is generally small, the lower regret for each selected feature ($\bigtriangleup_l$) in the attack leads to a lower regret bound of FEAT. Intuitively, FEAT tends to search for feasible perturbation within the highly sensitive features via the OMP-based feature ranking step. These features have a significantly lower $\bigtriangleup_l$ compared to the rest. Perturbing these features then help FEAT reduce the expected regret, which guarantees in theory a more successful attack after $T$ iterations of exploration. Compared to randomly sampling subsets of features in \cite{auer2002finite,Wang2008nips}, FEAT is destined to achieve a better balance between narrowing down the search range to save the computational overheads and maintaining the attack effectiveness. 

\section{Experimental Evaluation}
We use 4 categorical datasets with high-dimensional combinatorial search space ($N$ and/or $M$ are large) to measure the effectiveness and efficiency of FEAT. They are collected from diverse real-world applications, including text classification, cyber intrusion and digital health service. 

\noindent \textbf{Yelp-5} (\textbf{Yelp}) \cite{asghar2016yelp}. The Yelp-5 dataset was obtained from the Yelp Review Dataset Challenge in 2015. We use 650K training and 50K testing samples with the classes from 1 star to 5 stars for {training and testing} a classifier. Each word is encoded as categorical feature with a 300-dimensional embedding vector. 

\noindent \textbf{Intrusion Prevention System Dataset (IPS)} \cite{wang2020attackability}. Collected by a cyber-security vendor, the IPS dataset contains 242,467 instances of network attack reports. Each is composed by a sequence of 20 incident logs and categorised into any of the 3 threatening levels ('common', 'intermediate' and 'urgent'). At each log, the adversary may choose to replace it with 1,103 candidate logs. We randomly select 80\% of the IPS data for training and rest for testing.  

\noindent \textbf{Windows PE Malware Detection (PEDec)}. PEDec consists of dynamic analysis reports of 20,000 benignwares and 10,972 PE malwares. The malware samples of 152 families are randomly selected from those submitted to VirusTotal between 2018 and 2020. Each of the executables is classified as benign or malicious by more than 21 antivirus engines. 
In our work, each executable of the dataset is encoded into a binary feature vectors with 5000 signatures selected by human experts. 
We randomly select 80\% of each dataset for training and others for testing.

\noindent \textbf{Electronic Health Records (EHR)} \cite{MaKdd2018}.
The EHR dataset consists of time-ordered medical visit records of 7,314 patients. Each patient has from 4 to 200 medical visits. 
Each visit record is composed by a subset of 4,130 categorical ICD9 diagnosis codes\footnote{\url{http://www.icd9data.com/}}. Each diagnosis code represents occurrence of a disease, a symptom, or an abnormal finding. Using the historical EHR data of patients, our target is to predict whether a patient will suffer heart failure disease in the future. We randomly select 80\% of the EHR data for training and others for testing. 
Each EHR data instance is organized as a tensor $\mathbf{x} \in {R^{200*4130*70}}$ with each of the 4130 diagnosis codes projected to a 70-dimensional embedding vector. For the patients with less than 200 visits, we pad the empty observations by setting the corresponding $b^{j}_{i}=0$.

For each of IPS, EHR and PEDec datasets, we choose randomly 80\% of the dataset to train the target classifier. The rest 20\% of the data instances are used to test attack performances. On Yelp data, we choose 650k text instances for training the target text classification model. The rest 50k instances are used for performing different attacks and evaluating the attack performances. For \textit{PEDec} dataset, we adopt a simple CNN model composed of one convolution layer followed by two linear layers. The rest datasets contain sequential instances, we thus apply standard LSTM as the classifier. Without loss of generality, we use \textit{ReLu} activation function in both the CNN and LSTM classifier with the dropout module. We conduct all the experiments on Linux server with 2 GPUs (GeForce 1080Ti) and 16-core CPU (Intel Xeon). 
Implementations of the experiments are available at {\color{black}\url{https://github.com/xnudinfc/FEAT}}

\subsection{Performance Benchmark}
We include {three state-of-the-art domain-agnostic attack methods as the baselines.}  \\
\noindent \textbf{FSGS}~\cite{elenberg2018restricted}. FSGS is a greedy search-based method. In each round, it traverse the combination of each candidate feature with each subset of the already modified features. The candidate feature bringing the highest value of the attack objective function is selected to modify. FSGS only needs to query the target classifier $f$ to obtain the decision confidence. It is thus a \emph{black-box} attack method.\\
\noindent \textbf{OMPGS}~\cite{wang2020attackability}. It is a \emph{white-box} extension of FSGS by adopting the OMP-based ranking to constraint the greedy search within the top-ranked features in each round of the attack process. 
\\
\noindent \textbf{GradAttack}~\cite{QiSysML2018}. GradAttack is a \emph{white-box} evasion attack method originally proposed to generate adversarial text samples. 
It treats each word in a sentence as a categorical feature and uses gradients of the word embeddings to select feasible candidate words to attack. However, since it doesn't evaluate the combination patterns composed by the candidate words and the subset of words already modified, GradAttack requires to change much more features to deliver attacks than FSGS and OMPGS \cite{wang2020attackability}. 
\\
\noindent {\textbf{FEAT-B} is an variant of FEAT. It randomly samples $L$ out of the total $M$ features to conduct the UCB-guided search as in \cite{Wang2008nips}. Compared to FEAT, FEAT-B does not use the OMP computing to select the top influential features. \textit{The purpose of involving FEAT-B, as well as OMPGS, into the comparative study is to demonstrate the necessity of combing the OMP-based feature ranking and the UCB-based search together in FEAT.}}

We also involve two domain-specific attack methods \textbf{TextFooler} \cite{jin2020bert} and \textbf{TextBugger} \cite{li2018textbugger} into the test over the text data of Yelp-5. They have been popularly adopted in various attacks against text classification. We focus on showing FEAT, as an universally applicable method, can achieve similar applicability to the text-formatted categorical data, compared to these specially designed attack methods for NLP tasks based on semantic integrity/similarity-based knowledge. Demonstrating FEAT as a novel attack against text classification is beyond our scope. 

\textbf{Evaluation metrics.} 
{We measure the average number of the confidence computing operations required to deliver successful attack on one input instance (noted as \textbf{No.query}) as the metric of computational cost of the attack. In addition, we also show the averaged running time needed to attack one instance, noted as \textbf{Runtime} and measured in seconds. Both metrics indicate the computational efficiency level of each attack method. To evaluate the effectiveness of the attack, we use the attack success rate (\textbf{SR}) over all the testing instances. With the similar level of \textbf{SR}, if \textbf{Runtime} and \textbf{No.query} of an attack method are significantly lower than the other opponents, it indicates that this attack method could attack the high-dimensional instances faster.}

\vspace{-0.1cm}
\subsection{Attack Performance on Four Datesets}
\textbf{Computational Complexity Analysis}. We compare the computational complexity of all the involved domain-agnostic attack methods by counting the number of times evaluating the decision confidence $f({\hat{\mathbf{x}}})$ 
during attacking one input instance. In Table.~\ref{tab:query}, $T$ denotes the overall number of iterations required to reach successful attack. $N$ and $M$ are the number of features of 
$\mathbf{x}$ and category values of each $x_{i}$. $L$ denotes the number of selected features to explore in GradAttack, OMPGS, FEAT-B and FEAT. In each iteration, FSGS and OMPGS  only pick one feature to perturb. Therefore $T$ equals to the total number of changed features for these two methods. OMPGS, GradAttack and FEAT all use gradient information to shrink the size of the candidate features. Then, $L$ in the three methods denotes the number of the top-ranked candidate features selected by OMP. 

\begin{table}[t]
\centering
\linespread{1.2}
\caption{Complexity of the domain-agnostic attack methods}
\label{tab:query}
\footnotesize
\begin{tabular}{cc}
\toprule
\textbf{Attack Algo. }       & \textbf{Computational Complexity}   \\ \hline
FSGS          & $\sum_{t = 0}^T((N-t)*M*2^t)$      \\ 
GradAttack    & $T*\sum_{k = 0}^{{L}}(C^k_{{L}}*M^k)$      \\ 
OMPGS         & $\sum_{t = 0}^T({L}*2^t)$       \\ 
{FEAT-B}         & ${L}*{M} + {T}$       \\ 
{FEAT} & $({L}*{M} + {\tau})*{T}$       \\ \bottomrule
\end{tabular}
\vspace{-0.3cm}
\end{table}
{As given in Table.~\ref{tab:query}, the cost of FSGS and OMPGS increases as the sum of a geometric sequence of the number of changed features. In the high-dimensional problem with large $N$ and/or $M$, applying FSGS and OMPGS is prohibitively expensive. Similarly, GradAttack is expensive to conduct if $M$ becomes large. In contrast, the complexity of FEAT and FEAT-B is significantly lower, as $L\ll{N}$ (only the top $L$ ranked candidate features are considered) and it grows linearly as the number of the attack iterations increases. FEAT-B reduces to the standard UCB if $L$ equals to $N$. To initialize the exploration, the standard UCB needs to draw each candidate categorical feature at least once.  In the high-dimensional case, perturbing each feature once can induce an expensive overhead of $O(NM)$ to query the variation of the decision output of the target classifier by changing each feature. Compared to the standard UCB, the cost of FEAT is significantly lower.}


\begin{table*}[t]
\renewcommand{\arraystretch}{1.1}

\caption{Attack performances evaluated on \textbf{efficiency} and \textbf{effectiveness} metrics: The attack time limit  $T_L$=1000 sec. 
}
\vspace{-0.2cm}
\label{tab:Yelp-5 performance}
\scriptsize
\begin{subtable}[The results on Yelp-5 data]{
\resizebox{0.495\linewidth}{!}{
\begin{tabular}{lc|ccc}
\toprule
\multicolumn{2}{l|}{\textbf{Yelp-5}}                                                                                       & \multicolumn{3}{c}{\textbf{Budget = 6}}                              \\ \hline
\multicolumn{2}{l|}{\textbf{Attack Type \& Algo.}}                                                                                 & \textbf{Runtime} (sec) $\downarrow$   & \textbf{No.query} $\downarrow$  &\textbf{SR} $\uparrow$  \\ \hline
\multicolumn{1}{l|}{\multirow{2}{*}{\textbf{\begin{tabular}[c]{@{}l@{}}Domain\\ Specific\end{tabular}}}} & TextBugger       & 1.18          & 23                   & 0.64                  \\ \cline{2-5} 
\multicolumn{1}{l|}{}                                                                                    & TextFooler       & 0.42          & 167                    & 0.88                  \\ \hline
\multicolumn{1}{l|}{\multirow{2}{*}{\textbf{\begin{tabular}[c]{@{}l@{}}Black\\ Box\end{tabular}}}}       & FSGS            & 0.52          & 24000                   & 0.97                  \\ \cline{2-5} 
\multicolumn{1}{l|}{}                                                                                    & {FEAT-B} & {0.14} & {2057}  & {0.90}  \\ \hline
\multicolumn{1}{l|}{\multirow{3}{*}{\textbf{\begin{tabular}[c]{@{}l@{}}White\\ Box\end{tabular}}}}       & GradAttack      & 0.15          & 10000                  & 0.78                  \\ \cline{2-5} 
\multicolumn{1}{l|}{}                                                                                    & OMPGS           & 1.25          & 7000                    & 0.96                 \\ \cline{2-5} 
\multicolumn{1}{l|}{}                                                                                    & \textbf{FEAT}   & \textbf{0.14} & \textbf{887}   & \textbf{0.97}   \\ \bottomrule
\end{tabular}
\vspace{-2cm}
}
}\end{subtable}
\qquad
\begin{subtable}[The results on IPS data]{
\resizebox{0.495\linewidth}{!}{
\begin{tabular}{lc|ccc}
\toprule
\multicolumn{2}{l|}{\textbf{IPS}}                                                                                    & \multicolumn{3}{c}{\textbf{Budget = 5}}                                              \\ \hline
\multicolumn{2}{l|}{\textbf{Attack Type \& Algo.}}                                                                  & \textbf{Runtime} (sec) $\downarrow$   & \textbf{No.query} $\downarrow$      & \textbf{SR} $\uparrow$    \\ \hline
\multicolumn{1}{l|}{\multirow{2}{*}{\textbf{\begin{tabular}[c]{@{}l@{}}Black\\ Box\end{tabular}}}} & FSGS            & 136            & 37000                  & 0.80                 \\ \cline{2-5} 
\multicolumn{1}{l|}{}                                                                              & {FEAT-B} & {19.5} & {2500}  & {0.74}  \\ \hline
\multicolumn{1}{l|}{\multirow{3}{*}{\textbf{\begin{tabular}[c]{@{}l@{}}White\\ Box\end{tabular}}}} & GradAttack      & 21.2          & 2100        & 0.59       \\ \cline{2-5} 
\multicolumn{1}{l|}{}                                                                              & OMPGS           & 1.99           & 127       & 0.77        \\ \cline{2-5} 
\multicolumn{1}{l|}{}                                                                              & \textbf{FEAT}   & \textbf{0.28}  & \textbf{111}   & \textbf{0.92}  \\ \bottomrule
\end{tabular}
}
\vspace{-2cm}
}\end{subtable}
\\
\begin{subtable}[The results on PEDec data]{
\resizebox{0.495\linewidth}{!}{
\begin{tabular}{lc|ccc}
\toprule
\multicolumn{2}{l|}{\textbf{PEDec}}                                                                                  & \multicolumn{3}{c}{\textbf{Budget = 14}}                                             \\ \hline
\multicolumn{2}{l|}{\textbf{Attack Type \& Algo.}}                                                                  & \textbf{Runtime} (sec) $\downarrow$   & \textbf{No.query} $\downarrow$      & \textbf{SR} $\uparrow$    \\ \hline
\multicolumn{1}{l|}{\multirow{2}{*}{\textbf{\begin{tabular}[c]{@{}l@{}}Black\\ Box\end{tabular}}}} & FSGS            & 435           & 213256         & 0.88                   \\ \cline{2-5} 
\multicolumn{1}{l|}{}                                                                              & {FEAT-B} & {3.65} & {9959}  & {0.87}  \\ \hline
\multicolumn{1}{l|}{\multirow{3}{*}{\textbf{\begin{tabular}[c]{@{}l@{}}White\\ Box\end{tabular}}}} & GradAttack      & 3.51          & 18563             & 0.67                  \\ \cline{2-5} 
\multicolumn{1}{l|}{}                                                                              & {OMPGS}  & {360}  & {27758} & {0.80}  \\ \cline{2-5} 
\multicolumn{1}{l|}{}                                                                              & \textbf{FEAT}   & \textbf{2.89} & \textbf{5923} & \textbf{0.91}  \\ \bottomrule
\end{tabular}
}
\vspace{-2cm}
}
\end{subtable}
\qquad
\begin{subtable}[The results on EHR data]{
\resizebox{0.495\linewidth}{!}{
\begin{tabular}{lc|ccc}
\toprule
\multicolumn{2}{l|}{\textbf{EHR}}                                                                                    & \multicolumn{3}{c}{\textbf{Budget = 6}}                                                   \\ \hline
\multicolumn{2}{l|}{\textbf{Attack Type \& Algo.}}                                                                   & \textbf{Runtime} (sec) $\downarrow$   & \textbf{No.query} $\downarrow$    & \textbf{SR} $\uparrow$    \\ \hline
\multicolumn{1}{l|}{\multirow{2}{*}{\textbf{\begin{tabular}[c]{@{}l@{}}Black\\ Box\end{tabular}}}} & FSGS            & 482             & 58000           & 0.84                    \\ \cline{2-5} 
\multicolumn{1}{l|}{}                                                                              & {FEAT-B} & {167} & {7108}  & {0.94}  \\ \hline
\multicolumn{1}{l|}{\multirow{3}{*}{\textbf{\begin{tabular}[c]{@{}l@{}}White\\ Box\end{tabular}}}} & GradAttack      & 2.34            & 204              & 0.94                   \\ \cline{2-5} 
\multicolumn{1}{l|}{}                                                                              & OMPGS           & 27.5          & 35            & 0.94                   \\ \cline{2-5} 
\multicolumn{1}{l|}{}                                                                              & \textbf{FEAT}   & \textbf{0.35}   & \textbf{20}  & \textbf{0.94}   \\ \bottomrule
\end{tabular}
}
\vspace{-2cm}
}\end{subtable}
\vspace{-0.5cm}
\end{table*}




\textbf{Overall Performance}. We provide the detailed testing accuracy of the classifiers deployed on each dataset in Table.\ref{tab:classifier} of Appendix G. 
The results in Table.~\ref{tab:Yelp-5 performance}
illustrate that FEAT achieves generally both highly efficient computation (\emph{low \textbf{Runtime}} and \emph{low \textbf{No.query}}) and effective attack (\emph{high \textbf{SR}}), comparing with the other baseline attack methods on Yelp-5, IPS, EHR and PEDec. We organize the detailed comparison results with suitable  \textbf{attack budgets} (the maximum number of the modified features) on each dataset. Additional attack comparison with higher attack budget can be found in Table.\ref{tab:Yelp-5 performance V2} of Appendix.G. We highlight the performance metrics of the proposed FEAT with bold fonts in the followings.

{On Yelp-5 data, within the same attack budget, FEAT obtains very close \textbf{SR} level to those of the two domain-agnostic attack methods, FSGS and OMPGS. At the same time, FEAT's \textbf{No.query} and \textbf{Runtime}) are significantly lower than than those of FSGS and OMPGS, showing much higher efficiency for attack. The two greedy search methods (FSGS and OMPGS) exhaustively evaluate the combination of every candidate feature and the features that have been modified in previous iterations. In contrast, FEAT avoids the exhaustive search by balancing exploring rarely modified features and exploiting the features that show consistently high influence to the classifier's output in the search. 
The results validates that FEAT maintains attack effectiveness, while running in a much more efficient way. }


Compared to the domain-specific baselines (TextFooler and TextBugger), FEAT achieves 10$\%$ to 40$\%$ higher \textbf{SR} than those of TextFooler (0.97 v.s. 0.88), TextBugger (0.96 v.s. 0.64) and GradAttack (0.96 v.s. 0.78) on one hand. On the other hand, FEAT's \textbf{Runtime} is $33\%$  and $11\%$ of those of TextFooler and TextBugger respectively, while achieving higher \textbf{SR}. This indicates faster attack speed using FEAT. \textbf{No.query} of TextFooler and TextBugger is lower than that of \textbf{FEAT}. The reason is they incrementally modify words and evaluate the corresponding attack effects. On the contrary, FEAT evaluates all the candidate words at the initial step of the search, which is the origin of the increased computational overheads. However, via the initial per-word evaluation, FEAT can conduct the exploration in a more comprehensive way, which helps FEAT achieve much higher \textbf{SR}.

{\textbf{SR} of FEAT-B performs worse than baselines, e.g. the two greedy search-based methods, 
as randomly selecting features to explore is likely to miss influential features thus cause ineffective perturbation. The comparison between FEAT-B and FSGS/OMPGS shows that locating highly influential features to perturb is the key-to-success of attack. Nevertheless, FEAT-B always has orders of magnitude lower \textbf{Runtime} and \textbf{No.query} compared to FSGS. Both downsampling of the candidate feature set and conducting the UCB search help FEAT-B avoid exhaustive search in FSGS. The result 
implies the benefit of heuristically shrinking down the search range in the high-dimensional feature space. \textit{the key question to deliver simultaneously fast and effective attack is thus how to identify the most influential / sensitive features, where the UCB search is performed}. The OMP-boosted UCB search of FEAT answers this question and addresses the balance between the attack efficiency and effectiveness.}

On IPS and PEDec data, we can observe that FEAT consistently achieves the highest \textbf{SR}, significantly higher than the secondly ranked baseline. Meanwhile, \textbf{Runtime} and \textbf{No.query} of FEAT remain to be the lowest among all the attack methods. The results confirm the benefit of balancing exploration and exploitation in FEAT. On Yelp-5, IPS and PEDec data, GradAttack's \textbf{SR} is less than 80\% of that of FEAT. The reason is GradAttack requires more features to modify than the attack budget to deliver successful attacks over the testing inputs. Therefore, GradAttack is terminated when the number of modified features reaches the attack budget even before before it achieves successful attacks on the testing samples. 
On EHR data, with the similar level of the \textbf{SR} performance, FEAT obtains lower \textbf{Runtime} and \textbf{No.query} than other methods. Because of the sensitivity of the specific features in EHR, GradAttack and OMPGS can use the gradient of features to evaluate the importance of the features and select the most sensitive features very fast compared with the type of back-box adversarial attack. It is worth noting that the OMPGS and FEAT both use the OMP-based feature ranking to shrink the search range, the superior attack effectiveness and efficiency of FEAT confirm the merit of conducting the query-efficient UCB search over the top-ranked features, instead of the exhaustive greedy search in OMPGS. The total computational overheads of OMPGS and FEAT are composed of the cost for the OMP computation and the query evaluating the classifier’s output. OMPGS (greedy-based attack) needs to conduct significantly more OMP operations in each iteration than FEAT. Hence we can observe a much larger gap regarding the runtime measurement between OMPGS and FEAT, compared to that regarding the query number.

\vspace{-2mm}
\section{Discussion and Conclusion}

The proposed FEAT method explores how to deliver both effective and computationally efficient domain-agnostic adversarial attack in a high-dimensional categorical feature space. FEAT first conducts the orthogonal matching pursuit-based feature ranking to narrow down the search range to the most sensitive candidate features. After that, FEAT performs a MAB-driven combinatorial search over the shrinked set of candidate features. Through this way, FEAT maintains the effectiveness of the adversarial perturbation, while boosting the search efficiency to reach a fast yet successful attack. The comprehensive cross-application evaluation shows the superior domain-agnostic adaptivity of FEAT to different applications than the other state-of-the-art baselines, which makes FEAT a generally applicable tool to assess the adversarial risk of different applications with high-dimensional categorical inputs. However, FEAT still needs soft decision scores of the target classifier to evaluate different search paths. Perturbation-based defense, e.g. differential privacy, may help mitigate the attack. We will thus focus on the threat model with only hard labels accessible.



\vspace{-3mm}
\section*{Acknowledgements}
The research was partially supported by funding from King Abdullah University of Science and Technology (KAUST).

\newpage
\newpage
\bibliography{reference}

\begin{thebibliography}{44}
\providecommand{\natexlab}[1]{#1}

\bibitem[{Agrawal and Goyal(2013)}]{agrawal13aistats}
Agrawal, S.; and Goyal, N. 2013.
\newblock Further Optimal Regret Bounds for Thompson Sampling.
\newblock In Carvalho, C.~M.; and Ravikumar, P., eds., \emph{Proceedings of the
  Sixteenth International Conference on Artificial Intelligence and
  Statistics}, volume~31 of \emph{Proceedings of Machine Learning Research},
  99--107. Scottsdale, Arizona, USA: PMLR.

\bibitem[{Asghar(2016)}]{asghar2016yelp}
Asghar, N. 2016.
\newblock Yelp dataset challenge: Review rating prediction.
\newblock \emph{arXiv preprint arXiv:1605.05362}.

\bibitem[{Audibert, Munos, and Szepesv{\'a}ri(2007)}]{Audibert2007}
Audibert, J.-Y.; Munos, R.; and Szepesv{\'a}ri, C. 2007.
\newblock Tuning Bandit Algorithms in Stochastic Environments.
\newblock In Hutter, M.; Servedio, R.~A.; and Takimoto, E., eds.,
  \emph{Algorithmic Learning Theory}, 150--165. Berlin, Heidelberg: Springer
  Berlin Heidelberg.
\newblock ISBN 978-3-540-75225-7.

\bibitem[{Audibert, Munos, and Szepesv{\'a}ri(2009)}]{audibert2009exploration}
Audibert, J.-Y.; Munos, R.; and Szepesv{\'a}ri, C. 2009.
\newblock Exploration--exploitation tradeoff using variance estimates in
  multi-armed bandits.
\newblock \emph{Theoretical Computer Science}, 410(19): 1876--1902.

\bibitem[{Auer, Cesa-Bianchi, and Fischer(2002)}]{auer2002finite}
Auer, P.; Cesa-Bianchi, N.; and Fischer, P. 2002.
\newblock Finite-time analysis of the multiarmed bandit problem.
\newblock \emph{Machine learning}, 47(2): 235--256.

\bibitem[{Bao et~al.(2022)Bao, Han, Zhou, and Zhang}]{han2021iclr}
Bao, H.; Han, Y.; Zhou, Y.; and Zhang, X. 2022.
\newblock Towards understanding the robustness against evasion attack on
  categorical inputs.
\newblock In \emph{ICLR}.

\bibitem[{Biggio, Nelson, and Laskov(2012)}]{biggio2012poisoning}
Biggio, B.; Nelson, B.; and Laskov, P. 2012.
\newblock Poisoning attacks against support vector machines.
\newblock In \emph{Proceedings of the 29th International Coference on
  International Conference on Machine Learning}, 1467--1474.

\bibitem[{Buchbinder et~al.(2014)Buchbinder, Feldman, Naor, and
  Schwartz}]{buchbinder2014submodular}
Buchbinder, N.; Feldman, M.; Naor, J.; and Schwartz, R. 2014.
\newblock Submodular maximization with cardinality constraints.
\newblock In \emph{Proceedings of the twenty-fifth annual ACM-SIAM symposium on
  Discrete algorithms}, 1433--1452. SIAM.

\bibitem[{Campbell et~al.(2008)Campbell, Carmichael, Chai, Mena-Carrasco, Tang,
  Blake, Blake, Vay, Collatz, Baker, Berry, Montzka, Sweeney, Schnoor, and
  Stanier}]{Campbell2008science}
Campbell, J.~E.; Carmichael, G.~R.; Chai, T.; Mena-Carrasco, M.; Tang, Y.;
  Blake, D.~R.; Blake, N.~J.; Vay, S.~A.; Collatz, G.~J.; Baker, I.; Berry,
  J.~A.; Montzka, S.~A.; Sweeney, C.; Schnoor, J.~L.; and Stanier, C.~O. 2008.
\newblock Photosynthetic Control of Atmospheric Carbonyl Sulfide During the
  Growing Season.
\newblock \emph{Science}, 322(5904): 1085--1088.

\bibitem[{Carlini and Wagner(2018)}]{CarliniSP2018}
Carlini, N.; and Wagner, D. 2018.
\newblock Audio Adversarial Examples: Targeted Attacks on Speech-to-Text.
\newblock In \emph{SPW}.

\bibitem[{Cartella et~al.(2021)Cartella, Anunciacao, Funabiki, Yamaguchi,
  Akishita, and Elshocht}]{cartella2021adversarial}
Cartella, F.; Anunciacao, O.; Funabiki, Y.; Yamaguchi, D.; Akishita, T.; and
  Elshocht, O. 2021.
\newblock Adversarial Attacks for Tabular Data: Application to Fraud Detection
  and Imbalanced Data.
\newblock \emph{arXiv preprint arXiv:2101.08030}.

\bibitem[{Croce and Hein(2019)}]{Croce2019iccv}
Croce, F.; and Hein, M. 2019.
\newblock Sparse and Imperceivable Adversarial Attacks.
\newblock In \emph{ICCV}, 4723--4731.

\bibitem[{Ebrahimi et~al.(2018)Ebrahimi, Rao, Lowd, and
  Dou}]{ebrahimi2018hotflip}
Ebrahimi, J.; Rao, A.; Lowd, D.; and Dou, D. 2018.
\newblock HotFlip: White-Box Adversarial Examples for Text Classification.
\newblock In \emph{ACL}.

\bibitem[{Elenberg et~al.(2018)Elenberg, Khanna, Dimakis, and
  Negahban}]{elenberg2018restricted}
Elenberg, E.~R.; Khanna, R.; Dimakis, A.~G.; and Negahban, S. 2018.
\newblock Restricted strong convexity implies weak submodularity.
\newblock \emph{The Annals of Statistics}, 46(6B): 3539--3568.

\bibitem[{Gao et~al.(2018)Gao, Lanchantin, Soffa, and Qi}]{gao2018black}
Gao, J.; Lanchantin, J.; Soffa, M.~L.; and Qi, Y. 2018.
\newblock Black-box generation of adversarial text sequences to evade deep
  learning classifiers.
\newblock In \emph{2018 IEEE Security and Privacy Workshops (SPW)}, 50--56.
  IEEE.

\bibitem[{Goodfellow, Shlens, and Szegedy(2015)}]{Goodfellow2015}
Goodfellow, I.; Shlens, J.; and Szegedy, C. 2015.
\newblock Explaining and Harnessing Adversarial Examples.
\newblock In \emph{ICLR}.

\bibitem[{Goodfellow, Shlens, and Szegedy(2014)}]{goodfellow2014explaining}
Goodfellow, I.~J.; Shlens, J.; and Szegedy, C. 2014.
\newblock Explaining and harnessing adversarial examples.
\newblock \emph{arXiv preprint arXiv:1412.6572}.

\bibitem[{Hongyan et~al.(2022)Hongyan, Yufei, Yujun, Yun, and
  Xiangliang}]{hyhan2022iclr}
Hongyan, B.; Yufei, H.; Yujun, Z.; Yun, S.; and Xiangliang, Z. 2022.
\newblock Towards Understanding the Robustness Against Evasion Attack on
  Categorical Data.
\newblock In \emph{International Conference on Learning Representation (ICLR)}.
  Curran Associates, Inc.

\bibitem[{Imam and Vassilakis(2019)}]{imam2019survey}
Imam, N.~H.; and Vassilakis, V.~G. 2019.
\newblock A survey of attacks against twitter spam detectors in an adversarial
  environment.
\newblock \emph{Robotics}, 8(3): 50.

\bibitem[{Jia and Liang(2017)}]{jia2017adversarial}
Jia, R.; and Liang, P. 2017.
\newblock Adversarial examples for evaluating reading comprehension systems.
\newblock \emph{arXiv preprint arXiv:1707.07328}.

\bibitem[{Jin et~al.(2020)Jin, Jin, Zhou, and Szolovits}]{jin2020bert}
Jin, D.; Jin, Z.; Zhou, J.~T.; and Szolovits, P. 2020.
\newblock Is bert really robust? a strong baseline for natural language attack
  on text classification and entailment.
\newblock In \emph{Proceedings of the AAAI conference on artificial
  intelligence}, volume~34, 8018--8025.

\bibitem[{Li et~al.(2018)Li, Ji, Du, Li, and Wang}]{li2018textbugger}
Li, J.; Ji, S.; Du, T.; Li, B.; and Wang, T. 2018.
\newblock Textbugger: Generating adversarial text against real-world
  applications.
\newblock \emph{arXiv preprint arXiv:1812.05271}.

\bibitem[{Li et~al.(2020)Li, Ma, Guo, Xue, and Qiu}]{li2020bert-attack}
Li, L.; Ma, R.; Guo, Q.; Xue, X.; and Qiu, X. 2020.
\newblock Bert-attack: Adversarial attack against {BERT} using {BERT}.
\newblock In \emph{EMNLP}.

\bibitem[{Ma et~al.(2018{\natexlab{a}})Ma, Gao, Suo, You, Zhou, and
  Zhang}]{ma2018risk}
Ma, F.; Gao, J.; Suo, Q.; You, Q.; Zhou, J.; and Zhang, A. 2018{\natexlab{a}}.
\newblock Risk prediction on electronic health records with prior medical
  knowledge.
\newblock In \emph{Proceedings of the 24th ACM SIGKDD International Conference
  on Knowledge Discovery \& Data Mining}, 1910--1919.

\bibitem[{Ma et~al.(2018{\natexlab{b}})Ma, Gao, Suo, You, Zhou, and
  Zhang}]{MaKdd2018}
Ma, F.; Gao, J.; Suo, Q.; You, Q.; Zhou, J.; and Zhang, A. 2018{\natexlab{b}}.
\newblock Risk Prediction on Electronic Health Records with Prior Medical
  Knowledge.
\newblock In \emph{KDD}.

\bibitem[{Miyato, Dai, and Goodfellow(2016)}]{Miyato2016AdversarialTM}
Miyato, T.; Dai, A.~M.; and Goodfellow, I. 2016.
\newblock Adversarial Training Methods for Semi-Supervised Text Classification.
\newblock In \emph{ICLR}.

\bibitem[{Narodytska and Kasiviswanathan(2017)}]{Narodytska2017cvprw}
Narodytska, N.; and Kasiviswanathan, S. 2017.
\newblock Simple Black-Box Adversarial Attacks on Deep Neural Networks.
\newblock In \emph{2017 IEEE Conference on Computer Vision and Pattern
  Recognition Workshops (CVPRW)}, 1310--1318.

\bibitem[{Papernot et~al.(2016)Papernot, McDaniel, Swami, and
  Harang}]{papernot2016crafting}
Papernot, N.; McDaniel, P.; Swami, A.; and Harang, R. 2016.
\newblock Crafting adversarial input sequences for recurrent neural networks.
\newblock In \emph{MILCOM 2016-2016 IEEE Military Communications Conference},
  49--54. IEEE.

\bibitem[{Pendlebury et~al.(2019)Pendlebury, Pierazzi, Jordaney, Kinder, and
  Cavallaro}]{pendlebury2019tesseract}
Pendlebury, F.; Pierazzi, F.; Jordaney, R.; Kinder, J.; and Cavallaro, L. 2019.
\newblock {TESSERACT}: Eliminating experimental bias in malware classification
  across space and time.
\newblock In \emph{USENIX Security}.

\bibitem[{Pierazzi et~al.(2020)Pierazzi, Pendlebury, Cortellazzi, and
  Cavallaro}]{fabio2020sp}
Pierazzi, F.; Pendlebury, F.; Cortellazzi, J.; and Cavallaro, L. 2020.
\newblock Intriguing Properties of Adversarial ML Attacks in the Problem Space.
\newblock \emph{2020 IEEE Symposium on Security and Privacy}, 1332--1349.

\bibitem[{Qi et~al.(2019)Qi, Wu, P, A, Dhillon, and Witbrock}]{QiSysML2018}
Qi, L.; Wu, L.; P, C.; A, D.; Dhillon, I.; and Witbrock, M. 2019.
\newblock Discrete Attacks and Submodular Optimization with Applications to
  Text Classification.
\newblock In \emph{SysML}.

\bibitem[{Samanta and Mehta(2017)}]{samanta2017towards}
Samanta, S.; and Mehta, S. 2017.
\newblock Towards crafting text adversarial samples.
\newblock \emph{arXiv preprint arXiv:1707.02812}.

\bibitem[{Shu et~al.(2020)Shu, Mahudeswaran, Wang, Lee, and
  Liu}]{shu2020fakenewsnet}
Shu, K.; Mahudeswaran, D.; Wang, S.; Lee, D.; and Liu, H. 2020.
\newblock Fakenewsnet: A data repository with news content, social context, and
  spatiotemporal information for studying fake news on social media.
\newblock \emph{Big data}, 8(3).

\bibitem[{Stringhini, Kruegel, and Vigna(2010)}]{stringhini2010detecting}
Stringhini, G.; Kruegel, C.; and Vigna, G. 2010.
\newblock Detecting spammers on social networks.
\newblock In \emph{Proceedings of the 26th annual computer security
  applications conference}, 1--9.

\bibitem[{Suciu, Coull, and Johns(2019)}]{suciu2019exploring}
Suciu, O.; Coull, S.~E.; and Johns, J. 2019.
\newblock Exploring adversarial examples in malware detection.
\newblock In \emph{2019 IEEE Security and Privacy Workshops (SPW)}, 8--14.
  IEEE.

\bibitem[{Szegedy et~al.(2013)Szegedy, Zaremba, Sutskever, Bruna, Erhan,
  Goodfellow, and Fergus}]{szegedy2013intriguing}
Szegedy, C.; Zaremba, W.; Sutskever, I.; Bruna, J.; Erhan, D.; Goodfellow, I.;
  and Fergus, R. 2013.
\newblock Intriguing properties of neural networks.
\newblock \emph{arXiv preprint arXiv:1312.6199}.

\bibitem[{van Ede et~al.(2022)van Ede, Aghakhani, Spahn, Bortolameotti, Cova,
  Continella, van Steen, Peter, Kruegel, and Vigna}]{deepcase}
van Ede, T.; Aghakhani, H.; Spahn, N.; Bortolameotti, R.; Cova, M.; Continella,
  A.; van Steen, M.; Peter, A.; Kruegel, C.; and Vigna, G. 2022.
\newblock {DeepCASE: Semi-Supervised Contextual Analysis of Security Events}.
\newblock In \emph{IEEE S\&P}.

\bibitem[{Wang(2017)}]{wang2017liar}
Wang, W.~Y. 2017.
\newblock " liar, liar pants on fire": A new benchmark dataset for fake news
  detection.
\newblock \emph{arXiv preprint arXiv:1705.00648}.

\bibitem[{Wang, Audibert, and Munos(2008)}]{Wang2008nips}
Wang, Y.; Audibert, J.-y.; and Munos, R. 2008.
\newblock Algorithms for Infinitely Many-Armed Bandits.
\newblock In Koller, D.; Schuurmans, D.; Bengio, Y.; and Bottou, L., eds.,
  \emph{Advances in Neural Information Processing Systems}, volume~21. Curran
  Associates, Inc.

\bibitem[{Wang et~al.(2020)Wang, Han, Bao, Shen, Ma, Li, and
  Zhang}]{wang2020attackability}
Wang, Y.; Han, Y.; Bao, H.; Shen, Y.; Ma, F.; Li, J.; and Zhang, X. 2020.
\newblock Attackability Characterization of Adversarial Evasion Attack on
  Discrete Data.
\newblock In \emph{Proceedings of the 26th ACM SIGKDD International Conference
  on Knowledge Discovery \& Data Mining}, 1415--1425.

\bibitem[{Yang et~al.(2018)Yang, Chen, Hsieh, Wang, and
  Jordan}]{Yang2018GreedyAA}
Yang, P.; Chen, J.; Hsieh, C.; Wang, J.; and Jordan, M.~I. 2018.
\newblock Greedy Attack and Gumbel Attack: Generating Adversarial Examples for
  Discrete Data.
\newblock \emph{ArXiv}, abs/1805.12316.

\bibitem[{Yang et~al.(2020)Yang, Chen, Hsieh, Wang, and
  Jordan}]{yang2020greedy}
Yang, P.; Chen, J.; Hsieh, C.-J.; Wang, J.-L.; and Jordan, M.~I. 2020.
\newblock Greedy Attack and Gumbel Attack: Generating Adversarial Examples for
  Discrete Data.
\newblock \emph{J. Mach. Learn. Res.}, 21(43): 1--36.

\bibitem[{Zang et~al.(2020)Zang, Qi, Yang, Liu, Zhang, Liu, and
  Sun}]{zang2020word}
Zang, Y.; Qi, F.; Yang, C.; Liu, Z.; Zhang, M.; Liu, Q.; and Sun, M. 2020.
\newblock Word-level Textual Adversarial Attacking as Combinatorial
  Optimization.
\newblock In \emph{ACL}.

\bibitem[{Zhang and Combes(2021)}]{Zhang2021nips}
Zhang, R.; and Combes, R. 2021.
\newblock On the Suboptimality of Thompson Sampling in High Dimensions.
\newblock In Ranzato, M.; Beygelzimer, A.; Dauphin, Y.; Liang, P.; and Vaughan,
  J.~W., eds., \emph{Advances in Neural Information Processing Systems},
  volume~34, 8345--8354. Curran Associates, Inc.

\end{thebibliography}



\clearpage
\appendix

\section{A.Notations of used in the paper}
\label{sec:notation}
{We give the used notations in the following table.} 

\begin{table*}[h] 
\label{tab:notation}
\centering
\caption{The used notations}  
\begin{tabular}{ll} 
\toprule[2pt]
     & Notations of  instances \\
\midrule[1pt]
     
$\mathbf{x}$    & an instance\\
$N$    & the number of features for one instance\\
$M$    & the number of categorical values for each feature\\
$K$    & the number of the classes\\
$y_K$    & the ground truth label of one instance\\ 
$b^j_i$    & the binary indicator showing the presence of  the $j$-th categorical value in the $i$-th   feature  
  \\
$\mathbf{e}^j_i$    & the embedding of the categorical value of the categorical feature $i$ \\
$D$    & the dimension of  embedding \\

\midrule[1pt]
     & Notations used in algorithm FEAT\\
\midrule[1pt]
$\epsilon$    & the attack budget of attack \\
$T_L$   & the time limit of attack \\
$S$    & the set of features selected to attack the instance \\
$w_i$    & the weight of the feature's importance  \\
$L$    & the number of the top sensitive features selected by the weight $w_i$   \\
$G_{l,t^c}$    & the reward of modifying each $x_l$ in current $t^{c}$ iteration   \\
$B_{l}$    & the UCB score of the feature $x_l$   \\
${\Bar{\mu}}_{l,t}$    &  the empirical mean of the obtained rewards by changing $x_l$ after $t$ iterations of search.  \\
${\Bar{\delta}}^2_{l,t}$    &  the empirical variance of the obtained rewards by changing $x_l$ after $t$ iterations of search.  \\
$\hat{\mathbf{x}}$    & the modified instance \\
${R}_T$    & the regret of FEAT after $T$ iterations  \\
\midrule[1pt]
     & Notations of evaluation metrics in experiments\\
\midrule[1pt]
Runtime & the average running time in seconds\\
No.query & the average number of $f(\hat{\textbf{x}})$ evaluation\\
SR  & the  attack success rate \\
\toprule[2pt]
\end{tabular}
\end{table*}

\section{B.Proof of Theorem 1: Gradient as an indicator in the OMP strategy}

We define the regularization constraint over the classification function $f_{y}$ by extending \textit{Restricted Strong Convexity (RSC)} in Theorem.1 of  \cite{elenberg2018restricted} to apply to non-concave functions. The difference between ours and the theory proposed in \cite{elenberg2018restricted} is: their study is constrained to concave functions as classifiers, which facilitates the theoretical study. Given the target classifier is concave, the evasion attack process is a problem of strictly submodular function maximization, where the greedy search method can achieve a good approximated solution. Nevertheless, the unnatural concavity assumption could bring significant drop of the classification performances. In fact, it is rare to have a concave classification function deployed in practices. Our work addresses the gap by unveiling that attacking a classifier with a definite Lipschitz constant can be formulated as a problem weakly submodular maximization. Benefited from this, we can establish the guarantee that the gradient of the attack objective function can be considered as a bounded estimator of the marginal contribution of modifying each feature in the categorical input. Our theoretical study applies to much more broader classes of the target classifier, as most of the deployed Deep Learning models meet the Lipschitz condition.

\begin{definition}\label{def:boundness}
\textbf{Smoothness Condition of $f$.} Let $\Omega=({p},{q})$, ${p},{q}\in{\mathbb{R}^{n}}$ and $f$: $\mathbb{R}^{n}{\to}\mathbb{R}$ be a Lipschitz-continuous and differentiable function. A function $f$ is  $(m_{\Omega},M_{\Omega})$-\textit{smooth} on $\Omega$, if for any $({p},{q})\in\Omega$, $m_{\Omega}\in{\mathbb{R}}$ and $M_{\Omega} \in {\mathbb{R}^{+}}$,  $\epsilon = f({q})-f({p}) - \langle\nabla{f}({p}),{q}-{p}\rangle$ satisfies:
\begin{equation}\label{eq:smoothness}
\small
    \frac{m_{\Omega}}{2}\|{q}-{p}\|^2_{2} \leq |\epsilon| \leq  \frac{M_{\Omega}}{2}\|{q}-{p}\|^2_{2}.
\end{equation}
where $m_{\Omega}$ and $M_{\Omega}$ are the local strong convexity factor and local Lipschitz constant of the target classifier $f$ around the input $p$. 
\end{definition}

\begin{lemma}\label{theorem:submod_max}
Let $\Omega_{\zeta} = \{(\hat{{b}},\hat{{b}'}):|\text{diff}\,({b},\hat{{b}})|\leq{\zeta},|\text{diff}\,({b},\hat{{b}}')|\leq{\zeta},|\text{diff}\,(\hat{{b}},\hat{{b}}')|\leq{\zeta},\,\zeta\geq{1}\}$, where $\hat{{b}}$ and $\hat{{b}}'$ denote two sets of attribute changes. If the classifier $f_{y_k}({x})$ (k=1,...,$K$) follows the regularity condition given by $(m_{k,\Omega_{\zeta}},M_{k,\Omega_{\zeta}})$-\textit{smoothness} constraint on $\Omega_{\zeta}$, \textbf{the attack objective} in Eq.\ref{eq:adv_risk} can be formulated respectively as \textbf{monotone $\gamma_{\zeta}$-weakly submodular maximization}. 

Let {\footnotesize $\epsilon_{k} = f_{y_k}({x},\hat{{b}}') - f_{y_k}({x},\hat{{b}}) - \langle\nabla{f}_{y_k}({x},\hat{{b}}),\hat{{b}}'-\hat{{b}}\rangle$},  and {\footnotesize $\nabla{f}_{y}({x},{b})_{\nu}$} denote the elements of {\footnotesize $\nabla{f}_{y}({x},{b})$} corresponding to the difference between  $\hat{{b}}$ and $\hat{{b}}'$, where {\footnotesize $\nu = \textit{diff}(\hat{{b}},\hat{{b}}')$. }
The submodularity ratio $\gamma^{\textit{pess}}_{\zeta}$ for the pessimistic robustness assessment is bounded as:
\small
\begin{equation}\label{eq:wsub_max}
\small
\begin{split} 
    \gamma^{\textit{pess}}_{\zeta} = & \underset{k=1,...,K}{\min}\,\{\gamma^{\textit{pess}}_{k,\zeta}\}  
    \end{split}
\end{equation}
where {\footnotesize $\gamma_{k<K,\zeta}^{\textit{pess}}
    \geq\frac{\|\nabla{f}_{y_k}({x},{b})_{\nu}\|_{2}+m_{k,\Omega_{1}}|\zeta|/2}{\|\nabla{f}_{y_k}({x},{b})_{\nu}\|_{2}+M_{k,\Omega_{\zeta}}|\zeta|/2}$ } for {\footnotesize $\epsilon_{k}\geq{0}$, 
    $\gamma_{k<K,\zeta}^{\textit{pess}}
    \geq\frac{2m_{k,\Omega_{\zeta}}}{\|\nabla{f}_{y_k}({x},{b})_{\nu}\|^2_{2}}(\|\nabla{f}_{y_k}({x},{b})_{\nu}\|_{2}-M_{k,\Omega_{1}}|\zeta|/2)$ } for {\footnotesize $\epsilon_{k}<{0}$, 
    $\gamma_{K,\zeta}^{\textit{pess}} \geq \frac{2m_{K,\Omega_{\zeta}}}{\|\nabla{f}_{y_K}({x},{b})_{\nu}\|^2_{2}}(\|\nabla{f}_{y=K}({x},{b})_{\nu}\|_{2}-M_{K,\Omega_{1}}|\zeta|/2)$} for {\footnotesize $\epsilon_{K}\geq{0}$,}  and   {\footnotesize
    $\gamma_{K,\zeta}^{\textit{pess}}
    \geq\frac{\|\nabla{f}_{y_K}({x},{b})_{\nu}\|_{2}+m_{K,\Omega_{1}}|\zeta|/2}{\|\nabla{f}_{y_K}({x},{b})_{\nu}\|_{2}+M_{K,\Omega_{\zeta}}|\zeta|/2}$ } for {\footnotesize $\epsilon_{K}<{0}$.} 
\end{lemma}

Proof: the attack objective in Eq.\ref{eq:adv_risk} can be approximated as a set function optimization, which gives:
\begin{equation}\label{eq:pess_cerf}
\begin{split}
    &S^{*} = \max_{{S}}\{
    \underset{k=1,2,3,...,K-1}{\max}\{F_{{k}}(S)\} + (G_{{K}}(S))\}\\
    &F_{k}(S) = \max_{l\subset{S}}f_{{k}}(b_{l})\\
    &G_{K}(S) = \max_{l\subset{S}}\tilde{f}_{{K}}(b_{l})\\
\end{split}
\end{equation}
where $b_{l}$ denotes the modification of categorical features indicated by the index set $l$. $\tilde{f}_{K}(b_{l}) = -f_{K}(b_{l})$.
For each $k\in{1,2,3,...,K}$, $\epsilon_{k} = f_{k}({q})-f_{k}({p}) - \langle\nabla{f}_{k}({p}),{q}-{p}\rangle$. $F_{k}(S)$ and $G_{K}(S)$ are monotonically non-decreasing set functions. 

Supposing $\epsilon_{k} \geq {0}$ for each $k \in \{1,2,3,...,K-1\}$, we assume further that the features indicated by $l'$ are modified in addition to $l$, with $|l'|\leq{\zeta}$. For any $b_{l}$, $b_{l\cup{l'}}$ and $b_{l\cup{j}}$ denote the additional modification over the features indicated by $l'$ and by $j$ respectively to increase the output of $f_{k}$ (increasing the miss-classification decision confidence) and $-f_{K}$ (decreasing the confidence of correct classification). In the following analysis, we relax the discrete indicator $b$ to the continuous domain, as each $b_{i}\in[0,1]$. $\nabla{f}_{k}(b_{l})$ denotes the gradient of the classifier function $f_{k}$ with respect to the variable $b$ at $b=b_{l}$. Modifying from $b_{l}$ to $b_{l\cup{l'}}$ or $b_{l\cup{j}}$ follows the direction of gradient ascent, which gives $\langle \nabla{f}_{k}(b_{l}),b_{l\cup{l'}-b_{l}}\rangle \geq{0}$ and $\langle \nabla{f}_{k}(b_{l}),b_{l\cup{j}-b_{j}}\rangle \geq{0}$, $\langle \nabla{\tilde{f}}_{K}(b_{l}),b_{l\cup{l'}-b_{l}}\rangle \geq{0}$ and $\langle \nabla{\tilde{f}}_{K}(b_{l}),b_{l\cup{j}-b_{j}}\rangle \geq{0}$.
\begin{equation}\label{eq:concavity}
\begin{split}
    &f_{{k}}(b_{l\cup{l'}}) - f_{{k}}(b_{l}) \leq \langle\nabla{f}_{k}({b_{l}}),{b_{l\cup{l'}}}-{b_{l}}\rangle + \frac{M_{k,\Omega_\zeta}}{2}\|{b_{l\cup{l'}}}-{b_{l}}\|^2_{2} \\
    &\leq \|\nabla{f}_{k}({b_{l}})_{\zeta}\|_{2} + \frac{M_{k,\Omega_\zeta}}{2}|\zeta| \sum_{j\in{\zeta}}f_{{k}}(b_{l\cup{j}}) - f_{{k}}(b_{l})\\
    &\geq \sum_{j\in{\zeta}}\langle\nabla{f}_{k}({b_{l}}),{b_{l\cup{j}}}-{b_{l}}\rangle + \frac{m_{k,\Omega_1}}{2}\|b_{l\cup{j}}-b_{l}\|^2_2 \\
    &\geq \|\nabla{f}_{k}({b_{l}})_{\zeta}\|_{2} + \frac{m_{k,\Omega_1}}{2}|\zeta|\\
\end{split}
\end{equation}
We can derive the lower bound of the submodularity ratio $\gamma_{k,\zeta}$ of $f_{k}$:
\begin{equation}
    \gamma_{k,\Omega_\zeta} = \frac{\sum_{j\in{\zeta}}f_{k}(b_{l\cup{j}}) - f_{k}(b_{l})}{f_{k}(b_{l\cup{l'}}) - f_{k}(b_{l})} \geq \frac{\|\nabla{f}_{k}({b_{l}})_{\zeta}\|_{2} + \frac{m_{k,\Omega_1}}{2}|\zeta|}{\|\nabla{f}_{k}({b_{l}})_{\zeta}\|_{2} + \frac{M_{k,\Omega_\zeta}}{2}|\zeta|}
\end{equation}
Supposing $\epsilon_{k}<{0}$ for each $k\in\{1,2,3,...,K-1\}$, we can derive:
\begin{equation}\label{eq:convexity}
\begin{split}
     &f_{{k}}(b_{l\cup{l'}}) - f_{{k}}(b_{l}) \leq \langle\nabla{f}_{k}({b_{l}}),{b_{l\cup{l'}}}-{b_{l}}\rangle - \frac{m_{k,\Omega_\zeta}}{2}\|{b_{l\cup{l'}}}-{b_{l}}\|^2_{2}\\
    &\leq \frac{1}{2m_{k,\Omega_{\zeta}}}\|\nabla{f}_{k}(b_{l})_{\zeta}\|^2_{2}\\
    &\sum_{j\in{\zeta}}f_{{k}}(b_{l\cup{j}}) - f_{{k}}(b_{l}) \geq \sum_{j\in{\zeta}}\{\langle\nabla{f}_{k}({b_{l}}),{b_{l\cup{j}}}-{b_{l}}\rangle \\
    &- \frac{M_{k,\Omega_1}}{2}\|b_{l\cup{j}}-b_{l}\|^2_2\}\\
    &\geq \|\nabla{f}_{k}(b_{l})_{\Omega_\zeta}\|_{2} - \frac{M_{k,\Omega_1}}{2}|\zeta|
\end{split}
\end{equation}

Given the smoothness assumption on the targeted classifier (See Definition.\ref{def:boundness}), there exits a value of  $M_{k,\Omega_{1}}|\zeta|/2\leq{\|\nabla{f}_{k}(b_{l})_{{\zeta}}\|_{2}}$, which allows that $\|\nabla{f}_{k}(b_{l})_{\zeta}\|_{2} - \frac{M_{k,\Omega_1}}{2}|\zeta|\geq{0}$ holds. Therefore, we can derive the lower bound of the submodularity ratio of $F_{k}$ and $-G_{K}$, $\gamma_{k,\zeta}$ (k=1,2,3,...,K):
\begin{equation}
\begin{split}
    &\gamma_{k,\Omega_\zeta} = \frac{\sum_{j\in{\zeta}}f_{k}(b_{l\cup{j}}) - f_{k}(b_{l})}{f_{k}(b_{l\cup{l'}}) - f_{k}(b_{l})}\\ 
    &\geq\frac{2m_{k,\Omega_{\zeta}}}{\|\nabla{f}_{k}({b})_{\nu}\|^2_{2}}(\|\nabla{f}_{k}({b})_{\nu}\|_{2}-M_{k,\Omega_{1}}|\zeta|/2)\\
\end{split}
\end{equation}
The submodularity ratio $\gamma_{\zeta}$ on $\Omega_{\zeta}$ in Eq.\ref{eq:pess_cerf} is $\gamma_{\zeta}= \underset{k=1,2,3,...,K}{\min}{\{\gamma_{k,\Omega_{\zeta}}\}}$. 

As the attack objective in Eq.\ref{eq:adv_risk} is in nature a  problem of weakly submodular maximization, we combine Eq.\ref{eq:concavity} and Eq.\ref{eq:convexity} together to derive Eq.\ref{eq:gradientbound} in Theorem.1.

\section{C. Proof of Theorem.2: Regret Upper Bound for FEAT}
In \textit{FEAT}, conducting the OMP-boosted MAB searching process produces a set of features that are successively modified to reach the goal of attack. We provide below the proof to Theorem.2 in the paper, which provides the quality guarantee to the OMP-boosted MAB solution to the combinatorial optimization problem in Eq.\ref{eq:adv_risk}. The basic idea is to establish the regret gap of \textit{FEAT}: the smaller the regret gap is, the better the solution to the attack problem is and the better attack performances we may reach via the MAB search.

Eq.\ref{eq:UCBv} considers the mean and the variant of the {reward} of modifying each of the selected features in the search. . We have the following results established over the sampling complexity of the selected features.

\begin{lemma}\cite{audibert2009exploration}\label{lemma:Hoeffding_eq}
For the number of the selected $k$th feature  $t_k$ in $t$th iteration, $u = \lceil 8(\frac{\delta^2_k}{\bigtriangleup^2_k} + \frac{2}{\bigtriangleup_k})\log T\rceil$, $u \leq t_k\leq t \leq n$,$t\geq 2$ and $\tau_k \leq \bigtriangleup_k/2$, it holds that
\begin{equation}
   P(\hat{\mu}_k + \tau_k > \mu_l^* | t_k) \leq 2\exp({-t_k{\bigtriangleup_k}^2/(8\delta^2 + 4\bigtriangleup_k/3)}).
\end{equation}
\end{lemma}

The Lemma \ref{lemma:Hoeffding_eq}. shows that at the beginning iteration $\hat{\mu}_{k} < \mu^*$. With the increasing of the number for selecting feature $k$, that is $t_k$, the probability for selecting the sub-optimal feature $k$ will be reducing exponentially without the relation to $t$ and $T$.
After pulling $t_k$ times feature $k$, the upper confidence bound $\hat{\mu}_k + \tau_k$ larger than optimal feature reward $\mu^*$ with the probability at most $\exp({-t_k{\bigtriangleup_k}^2/(8\delta^2 + 4b\bigtriangleup_k/3)}$. We want to realize the probability to 0 when $t$ is large enough. Here we choose $P = {t}^{-\alpha}$, we will obtain:
\begin{align}
    \tau_k = \sqrt{\frac{\alpha{\delta^2}_{k,t_k}\log{t} }{t_k}} + \frac{\alpha^2\log{t}}{t_k}
\end{align}

\begin{lemma}\label{lemma:three conditions}
Let $k^*$ denote the any optimal feature (which means $\mu^* = \mu_k^*$) and suppose that FEAT selects the feature $I_t =k$ at iteration $t$ ($\bigtriangleup_k >0$). The  at least one of the following three statements is true:
\begin{align}
(a)~~~ &\hat{\mu}^* \leq \mu^* - \tau_l^* \\
(b)~~~ &\hat{\mu}_k \geq \mu_k + \tau_k  \\
(c)~~~ & n_k < 8(\frac{\delta^2_k}{\bigtriangleup^2_k} + \frac{2}{\bigtriangleup_k})\log T
\end{align}
\end{lemma}

\begin{proof}
Suppose these three statements (a), (b), (c) are false. Then,
\begin{align}
\hat{\mu}^* + \tau_l^* & > \mu^* = \mu_k + \bigtriangleup_k&(a)false\\
& \geq \mu_k + 2(\sqrt{\frac{ \alpha {\delta^2}_{k,t_k}*\log{T} }{t_k}} + \frac{\log{T}}{t_k})&(c)false\\
& \geq \mu_k + 2(\sqrt{\frac{ \alpha {\delta^2}_{k,t_k}*\log{t} }{t_k}} + \frac{\log{t}}{t_k}) &\log T \downarrow \\
& > \hat{\mu}_k + \sqrt{\frac{ \alpha {\delta^2}_{k,t_k}*\log{t} }{t_k}} + \frac{\log{t}}{t_k}&(b) false.
\end{align}
So we obtain:
\begin{align}
\hat{\mu}^* + \tau_l^* & > \hat{\mu}_k + \sqrt{\frac{ \alpha {\delta^2}_{k,t_k}*\log{t} }{t_k}} + \frac{\log{t}}{t_k}
\end{align},
which is contradicted to the assumption that FEAT selects the sub-optimal feature $k$ rather than the optimal feature $l^*$.
\end{proof}

In the latter proof, we denote the total number of discrete features in one input instance as $L$, instead of $K$, to keep consistent notations as in our submission. 

\begin{theorem}
\label{theorem:regert bound_appendix}
\textbf{Perturbing highly sensitive features helps shrink the regret bound of FEAT.}  Let $\bigtriangleup_l >0$ and ${\delta^2_{l}}>0$ be the expected regret and the expected variance of the rewards received by modifying each of 
the top-$L$ candidate features. 
The expected regret bound of FEAT after $T$ iterations can be given as:
\vspace{-0.1cm}
\begin{equation}\label{eq:regretbound_appendix} 
\small
    \EX[{R}_T] \; \leq  \; \sum_{l=1}^L [8(\frac{\delta^2_{l}}{\bigtriangleup_l} + 2)\log T + \frac{\alpha}{\alpha-2} \bigtriangleup_l], \vspace{-0.01cm}
\end{equation}    \vspace{-0.02cm}
where $\EX[{R}_T] \overset{\text{def}}{=} \sum_{l=1}^L \EX[T_l]  \bigtriangleup_l$. ${T_l}$ is the number of times that feature $l$ is selected after $T$ iterations of the UCB-guided search. 
\end{theorem}


    

\begin{proof}
From the $\EX[{R}_T] \overset{\text{def}}{=} \sum_{l=1}^L \EX[T_l]  \bigtriangleup_l$, we know, in order to bound the expection of the regret we have to bound the $\EX[T_l]$ firstly. We divide the features into two parts, any optimal features $l^*$ and sub-optimal feature $l$. Let $u = \lceil 8(\frac{\delta^2_l}{\bigtriangleup^2_l} + \frac{2}{\bigtriangleup_l})\log T\rceil$. Then

\begin{align}
\EX[T_l]  &= \EX[\sum_{t = 1}^{T}\mathbf{1}(I_t =l)]\\
          &= \EX[\sum_{t = 1}^{T}\mathbf{1}(I_t =l,T_l < u) + \sum_{t = 1}^{T}\mathbf{1}(I_t =l,T_l \geq u)]\\
          &\leq u + \sum_{t = u+1}^{T}P(I_t =l,T_l \geq u)\\
          &\leq u + \sum_{t = u+1}^{T}(1-P(I_t =l,T_l < u))\\
          &\leq u + \sum_{t = u+1}^{T}[P(\hat{\mu}^* \leq \mu^* - \tau_l^*) + P(\hat{\mu}_l \geq \mu_l + \tau_l)].
\end{align}

From the Lemma \ref{lemma:Hoeffding_eq}, we get 
\begin{align}
P(\hat{\mu}_l + \tau_l > \mu^*) &= P(\hat{\mu}_l + \tau_l >  \mu_l + \bigtriangleup_l) \\
&\leq P(\hat{\mu}_l > \mu_l + \bigtriangleup_l/2)\\
&\leq 2\exp({-T_l{\bigtriangleup_l}^2/(8\delta^2 + 4\bigtriangleup_l/3)})
\end{align}

So the 
\begin{align}
\sum_{t = u+1}^{T}P(\hat{\mu}^* \leq \mu^* - \tau_l^*) &\leq \sum_{t = u+1}^{T}P(\prod_{T_l^* = 1}^{t-1}\hat{\mu}^* \leq \mu^* - \tau_l^*)\\
&\leq \sum_{t = u+1}^{T}\frac{1}{t^{\alpha-1}}\\
&\leq \sum_{t = 2}^{T}\frac{1}{t^{\alpha-1}}\\
&\leq \frac{1}{\alpha-2},
\end{align}
here we use the in-equation in Lemma \ref{lemma:Hoeffding_eq}. With the similar way, we obtain $P(\hat{\mu}_l \geq \mu_l + \tau_l) \leq \frac{1}{\alpha-2}$.

Therefore, we can bound the $\EX[T_l]$ to :
\begin{align}
\EX[T_l] &\leq 8(\frac{\delta^2_l}{\bigtriangleup^2_l} + \frac{2}{\bigtriangleup_l})\log T +1 + \frac{2}{\alpha-2} \\
&\leq 8(\frac{\delta^2_l}{\bigtriangleup^2_l} + \frac{2}{\bigtriangleup_l})\log T + \frac{\alpha}{\alpha-2}
\end{align}

The expected regret bound is:
\begin{align}
    \EX[{R}_T] &= \sum_{l=1}^L \EX[T_l]  \bigtriangleup_l\\
    &\leq \sum_{l=1}^L [8(\frac{\delta^2_l}{\bigtriangleup_l} + 2)\log T + \frac{\alpha}{\alpha-2} \bigtriangleup_l].\\
\end{align}
\end{proof} 




\section{D. Feature sensitivity analysis for different datasets}

The proposed OMP-boosted MAB search in \textit{FEAT} improves the efficiency of search by identifying the highly sensitive features first and narrowing down the exploration range biased towards these features after. In the following empirical study, we show how the feature sensitivity level varies across  features in different datasets. 

{We define the metric evaluating the feature sensitivity level (FS) as the change of the decision confidence of the target classifier if feature $i$ of the input instance is changed.
\begin{equation}\label{eq:sensitivity}
    FS_{i} = f_{y_{k}}(\hat{\mathbf{x}}_{-i}) - f_{y_{k}}(\mathbf{x})
\end{equation}
where $\hat{\mathbf{x}}_{-i}$ denotes the modified input instance with only the feature $i$ changed. $\mathbf{x}$ is the original input instance. The larger $FS_{i}$ is, the more sensitive the corresponding feature $i$ is in the classification task. We aim to identify such highly sensitive features. Compared to the less sensitive ones, conducting adversarial perturbations over these features is likely to cause more drastic variation to the target classifier's output.}
 
We present the feature sensitivity for evaluation datasets except Yelp-5 due to its large vocabulary size. The reported feature sensitivity of one feature is   the average change of the prediction level  in all instances.

Fig \ref{fig:EHR} and Fig \ref{fig:IPS-sensitivity} show that only a few features in EHR and IPS dataset have significantly high sensitivity level in the classification tasks. In EHR, the feature 31 has the average sensitivity value 0.24, whereas most of the others are lower than 0.04. In IPS, the feature 19 has the average sensitivity level larger than 0.5, which means attacking only the feature 19 can already manage to achieve the attack goal successfully. The sensitivity of the other features in IPS are less than 0.2, which are significantly less useful in the attack.  

However, Fig.\ref{fig:PEDec} shows that the feature sensitivity level of PEDec across 5,000 malware features has a  small variance (the feature sensitivity ranges from 0.435 to 0.502). It implies that   malware features may contribute equally to the adversarial attack.

\begin{figure}[htbp]
\centering
\begin{minipage}[t]{0.4\textwidth}
\includegraphics[width=0.9\textwidth]{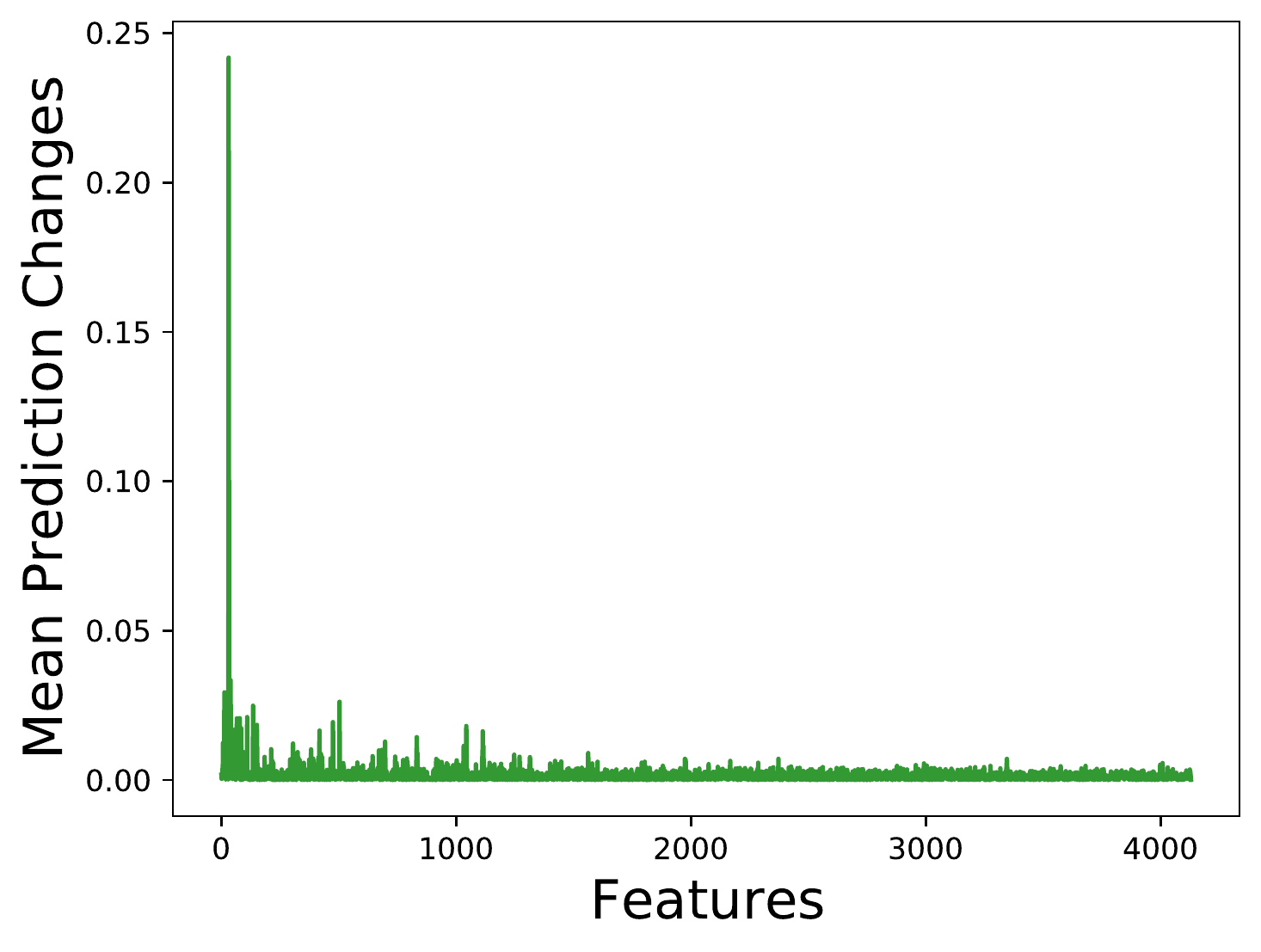}
\caption{Feature sensitivity of EHR dataset}
\vspace{+0.5cm}
\label{fig:EHR}
\end{minipage}
\vspace{+0.5cm}
\begin{minipage}[t]{0.4\textwidth}
\centering
\includegraphics[width=0.9\textwidth]{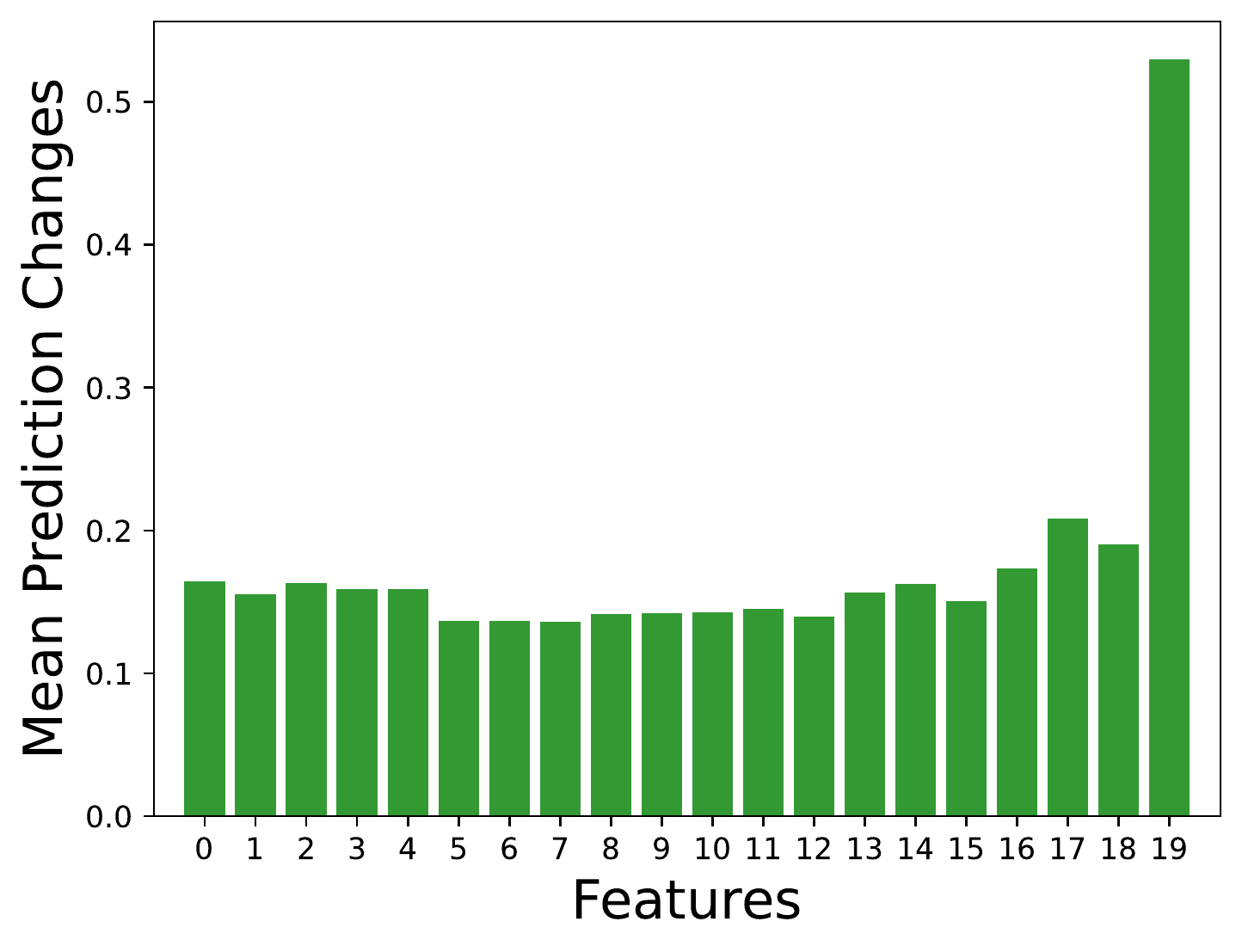}
\caption{Feature sensitivity of IPS dataset}
\label{fig:IPS-sensitivity}
\end{minipage}
\begin{minipage}[t]{0.4\textwidth}
\centering
\includegraphics[width=0.9\textwidth]{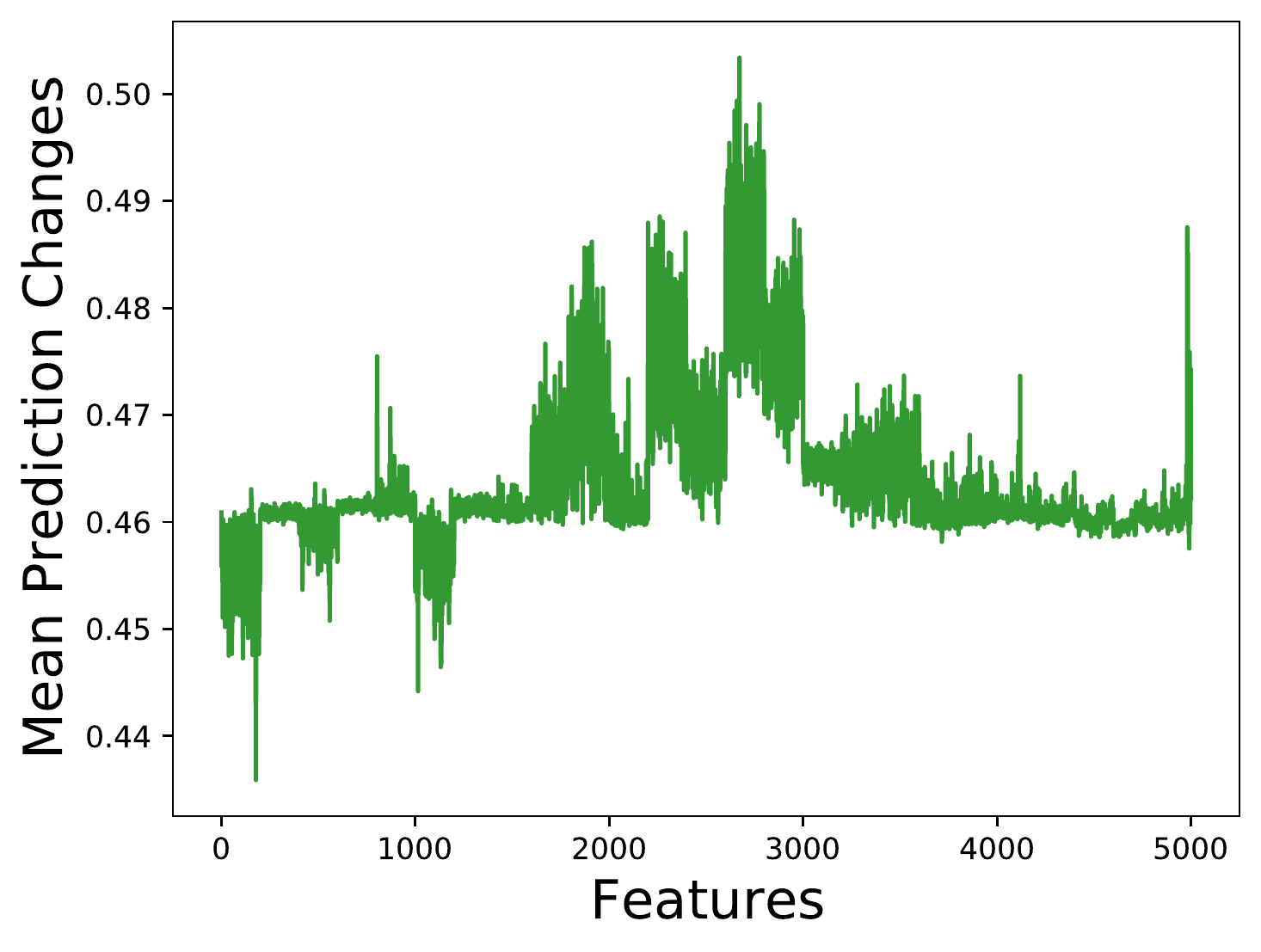}
\caption{Feature sensitivity of PEDec dataset}
\label{fig:PEDec}
\end{minipage}
\end{figure}


\section{E.Parameter sensitivity analysis}


We show below the variation of the attack performances using \textit{FEAT}, with different settings of $\alpha$ in Eq.\ref{eq:UCBv}. Adjusting the value of $\alpha$ balance exploration and exploitation in the search process. Larger $\alpha$ prefers exploration over the features that are not visited and modified before. On the contrary, smaller $\alpha$ biases the search process within the features that are considered to contribute high rewards.

\begin{table*}[]
\centering
\renewcommand{\arraystretch}{1.3}
\caption{Attack performances of FEAT over the four datesets with different choices of $\alpha$}
\label{tab: Exploration and Exploitation}
\resizebox{0.9\linewidth}{!}{
\begin{tabular}{cc|cccccccc}
\toprule
\multicolumn{2}{c|}{\textbf{Dateset}}                                     & \multicolumn{2}{c}{\textbf{Yelp-5}} & \multicolumn{2}{c}{\textbf{IPS}} & \multicolumn{2}{c}{\textbf{EHR}} & \multicolumn{2}{c}{\textbf{PEDec}} \\ \hline
\multicolumn{2}{c|}{\textbf{Budget}}                                      & \textbf{6}       & \textbf{8}       & \textbf{5}      & \textbf{6}     & \textbf{4}      & \textbf{6}     & \textbf{10}      & \textbf{12}     \\ \hline
\multicolumn{1}{c|}{\multirow3{*}{\textbf{$\alpha$ = 0}}} & \textbf{Runtime} (sec) & 0.11             & 0.24             & 0.56            & 0.74           & 0.86            & 0.38           & 8.98             & 9.36            \\ \cline{2-10} 
\multicolumn{1}{c|}{}                                       & \textbf{No.query} & 1293             & 2271             & 352             & 509            & 82.41           & 29.66          & 22618            & 23839           \\ \cline{2-10} 
\multicolumn{1}{c|}{}                                       & \textbf{SR} & 0.893            & 0.945            & 0.81            & 0.82           & 0.89            & 0.84           & 0.73             & 0.82            \\ \hline
\multicolumn{1}{c|}{\multirow{3}{*}{\textbf{$\alpha$ = 2}}} & \textbf{Runtime} (sec) & 0.103            & 0.144            & 0.29            & 0.28           & 0.75            & 0.33           & 2.39             & 2.41            \\ \cline{2-10} 
\multicolumn{1}{c|}{}                                       & \textbf{No.query} & 662              & 911              & 113             & 112            & 71.28           & 24.70          & 7080             & 6607            \\ \cline{2-10} 
\multicolumn{1}{c|}{}                                       & \textbf{SR} & 0.95             & 0.96             & 0.92            & 0.93           & 0.89            & 0.899          & 0.75             & 0.87            \\ \hline
\multicolumn{1}{c|}{\multirow{3}{*}{\textbf{$\alpha$ = 4}}} & \textbf{Runtime} (sec) & 0.101            & 0.143            & 0.28            & 0.28           & 0.74            & 0.32           & 4.10             & 4.13            \\ \cline{2-10} 
\multicolumn{1}{c|}{}                                       & \textbf{No.query} & 643              & 883              & 111             & 121            & 71.26           & 24.52          & 7080             & 6607            \\ \cline{2-10} 
\multicolumn{1}{c|}{}                                       & \textbf{SR} & 0.96             & 0.965            & 0.92            & 0.94           & 0.89            & 0.897          & 0.75             & 0.87            \\ \hline
\multicolumn{1}{c|}{\multirow{3}{*}{\textbf{$\alpha$ = 8}}} & \textbf{Runtime} & 0.1              & 0.12             & 0.25            & 0.28           & 0.73            & 0.34           & 3.51             & 3.12            \\ \cline{2-10} 
\multicolumn{1}{c|}{}                                       & \textbf{No.query} & 632              & 883              & 119             & 123            & 54.59           & 20.45          & 7080             & 6607            \\ \cline{2-10}
\multicolumn{1}{c|}{}                                       & \textbf{SR} & 0.954            & 0.965            & 0.93            & 0.92           & 0.932           & 0.935          & 0.75             & 0.87            \\ \bottomrule
\end{tabular}
}
\end{table*}

From Table.\ref{tab: Exploration and Exploitation}, we represent the performance metrics with $\alpha = 0, 2, 4, 8$ over four datasets. The \textbf{bold set} in each dataset have the rather better performance considering \textbf{Runtime}, \textbf{No.query}, \textbf{SR} together. PEDec (all the features sensitivity are almost the same level) need the smallest $\alpha$ to adjust the exploration, which means every feature you select in each iteration will contribute for the final reward and which dose not need deep exploration. Even you take the deep exploration $\alpha = 8$, the attack performance results are almost the same as the condition $\alpha = 2$.
However, EHR and IPS (only one or a few features have the high features sensitivity) are suitable with the larger $\alpha =$ 4 and 8. Because the only one or a few features can make the main contribution to the adversarial attack, we explore these features in the whole features in the searching process. Especially if the whole feature number is very big, such as EHR with 4,130 features, this exploration will be more difficult.

Globally, on the datasets of severely skewed feature sensitivity distributions (like EHR and IPS), we prefer a large $\alpha$ to encourage exploration in the multi-armed bandit search over all of the features, so as to extend the exploration range to cover the rarely appearing yet highly sensitive features to guarantee the success of attack. In contrast, on the datasets of the uniformly distributed feature sensitivity level (like PEDec), we 
consider to choose a small $\alpha$ to put more weights on exploitation over the features empirically inducing larger variations than the others over the classifier's decision confidence scores, in order to reach the attack goal within as few as possible iterations, i.e. improving the efficiency of the attack. 



\section{F.Explanation to the adopted MAB setting}

\textbf{First}, we show that adversarial attack over a categorical input $x_{i}$ is intrinsically a problem of set function maximization, following Eq.1 and Eq.4 on page 1 and 4 of \cite{hyhan2022iclr}. 

Given an input instance ${x} = \{x_1,x_2,x_3,...,x_N\}$ composed of $N$ categorical attributes, each $x_i$ can take any of optional ${M}$ categorical values. The classifier $f$ produces the soft decision scores $f_{y_k}$ ($k=1,2,3,...,{K}$) with respect to $K$ class labels. 
The goal of the attack is to find the minimal set of categorical feature perturbations $l= \texttt{diff}(x,\hat{{x}})$, with which the maximum gap between the decision score of any wrong label $k$ and the correct label ${K}$ is larger than 0. 
\begin{equation}\label{eq:optimal_attack}
\small
\begin{split}
&\hat{x}^{*} = \underset{\hat{{x}},\,\,l= \texttt{diff}({x},\hat{{x}})}{\argmax} \quad m_f  \\
&s.t. \quad m_f \geq{0},\,\;\;|l|\leq\varepsilon 
\end{split}
\end{equation}
where $m_f=\underset{k\in\{1,...,{K}-1\}}{\max}\{f_{y_k}(\hat{x})\} - f_{y_{K}}(\hat{x})$ is the gap of the decision scores over the perturbed input instance $\hat{x}$. $m_{f}\in[-1,0)$ and $m_{f}\geq{0}$ correspond to correct and wrong classification output respectively. $\texttt{diff}$ denotes the set of the perturbed categorical features in in the input instance $x$. $|\texttt{diff}|\leq{\epsilon}$ gives the budget limit of the attack. \textit{The number of the modified features in the attack should be no more than $\epsilon$}. Different from the attacks with continuous data, solving the set function optimization problem requires to conduct a combinatorial search in the categorical feature space. This is in nature NP-hard. In the high-dimensional case, performing exhaustive search to solve Eq.1 is prohibitively expensive and randomly selecting features to perturb can produce arbitrarily bad results. 

\noindent \textbf{Relation between the reward definition and the set function maximization problem.} In our study, the reward of modifying a categorical feature $x_l$ in the current iteration $t^c$ of the exploration process is defined as the maximum possible gap $m_{f}$ between any wrong label $k$ and the correct label $K$. We reorganize the formal definition of the reward value as in Eq.2. 
\begin{equation}\label{eq:reward_appendix}
\small
\begin{split}
    G_{l,t^c} &= \underset{k=1,2,3,...,K-1}{\max} f_{y_{k}}(\hat{\mathbf{x}}_{t^{c}}) - f_{y_{K}}(\hat{\mathbf{x}}_{t^{c}}) + \Lambda \\
    & \texttt{diff}(\hat{\mathbf{x}}_{t^{c}},\hat{\mathbf{x}}_{t^c-1}) = \{x_{l}\}\\
\end{split}
\end{equation}
where $\hat{\mathbf{x}}_{t^{c}}$ denotes the adversarially perturbed input instance at the current iteration $t^{c}$. $\texttt{diff}(\hat{\mathbf{x}}_{t^{c}},\hat{\mathbf{x}}_{t^c-1}) = \{x_{l}\}$  denotes that only one more feature $x_l$ is changed at the current iteration $t^c$, compared to the previous iteration $t^c-1$. We add a large enough constant $\Lambda$, e.g. $\Lambda=1$ in practices, to ensure the non-negativeness of the received rewards. 

The difference between the maximum reward derived at the two successive iterations $t^{c}$ and $t^{c}-1$ equals to the marginal gain of the set function maximization problem in Eq.1 by incrementally modifying one more feature.  
\begin{equation}
\small
\begin{split}
    &G^{*}_{t^c} - G^{*}_{t^c-1} = m_{f}(\hat{x}^{*}_{t^c}) - m_{f}(\hat{x}^{*}_{t^c-1}) \\
    &G^{*}_{t^c} = \underset{x_{l}\in S^{t^c}}{max}\,\,G_{l,t^c}\\
    &G^{*}_{t^c-1} = \underset{x_{l}\in S^{t^c-1}}{max}\,\,G_{l,t^c-1}
\end{split}
\end{equation}
where $S^{t^c}$ and $S^{t^c-1}$ are the sets of the candidate features at the iteration $t^c$ and $t^c-1$. $m_{f}(\hat{x}^{*}_{t^c})$ and $m_{f}(\hat{x}^{*}_{t^c})$ are the maximum $m_{f}$ value reached at the current iteration $t^{c}$ and its precedent iteration $t^c-1$. The difference between two maximum $m_{f}$ value defines the marginal gain of the maximization objective in Eq.1 by incrementally perturb one more feature of the input instance $x$. 

\noindent \textbf{Observation.1} The difference between the maximum rewards achieved at two successive iterations (in Eq.3) does not necessarily diminish for a general classifier $f$. 

As a set function optimization problem, \textit{the marginal gain diminishes only if the set function-based objective follows the strict submodularity property} \cite{elenberg2018restricted}, which further requires all the parameters of the target classifier to be positive \cite{QiSysML2018,elenberg2018restricted}. However, the positiveness constraint over the classifier's parameters is not realistic and restricts severely the flexibility of the classifier. For a general classifier $f$, the submodularity of the corresponding attack objective (in Eq.1 of the revised submission) thus does not hold. As a result, the difference between the rewards received at successive iterations does not necessarily converge to zero along the search process. On the contrary, the reward difference given in Eq.45 can vary in a non-monotonically way. 

This observation further implies the non-diminishing property of the reward value in the attack scenario against a general classifier. The reward difference given in Eq.45 can be considered as a relative effectiveness measurement of the perturbation efforts introduced at each iteration, compared to the last iteration. The larger the reward difference is, the more effective perturbation the current search iteration introduces compared to the previous iteration. In this sense, the non-monotonic variation of the reward difference indicates that the perturbation effects injected by modifying one more candidate feature does not diminish monotonically as the attack objective increases.

\noindent \textbf{Observation.2} In the long run, we acknowledge that the reward distribution of changing one feature $x_{i}$ in the input is not stationary in general along the combinatorial search process. However, within the successive few iterations, the reward of modifying the top sensitive features does not commit significant drift. The rewards received by modifying these features remain stable within a few iterations of the combinatorial exploration. 

For a classifier $f$ and a data set of testing input instances, we conduct one-factor-at-a-time sensitivity analysis \cite{Campbell2008science} over the categorical features. Given an input instance $x$, we change each feature $x_{i}$ while keeping all the others fixed. The averaged change of the probabilistic classification output over all the input instances in the data set is used as the feature-wise sensitivity measurement. A larger average change indicates that the classifier’s output is more sensitive to the change over the corresponding categorical feature. 

Given the current iteration $t^c$, we measure the reward values received by modifying the top 10 sensitive features within the consecutive iterations $\{t^c+T\}$ ($T=0,1,2,3,4,...,6$). We then compute the ratio between the standard deviation and the average of the reward values by modifying each of the top 10 sensitive features $x_{l}$ within the consecutive iterations, noted as $\texttt{var}(\{G_{l,t^c+T}\})/\texttt{avg}(\{G_{l,t^c+T}\})$. The smaller the ratio $\texttt{var}(\{G_{l,t^c+T}\})/\texttt{avg}(\{G_{l,t^c+T}\})$ is, the more stable the reward of the corresponding feature $x_{l}$ stays across the consecutive iterations. We observe that all of the ratio values of the top 10 sensitive features remains lower than $1e^{-2}$ on the 4 datasets. In Figure.4, we take PEDec and EHR as the examples and show the value of $\texttt{var}(\{G_{l,t^c+T}\})/\texttt{avg}(\{G_{l,t^c+T}\})$ for the top-10 sensitive features on PEDec within the consecutive 6 UCB-based search iterations. 

\begin{figure}[t]\label{fig:RewardDistribution}
\centering
\begin{minipage}[t]{0.4\linewidth}
\centering
\includegraphics[width=\textwidth]{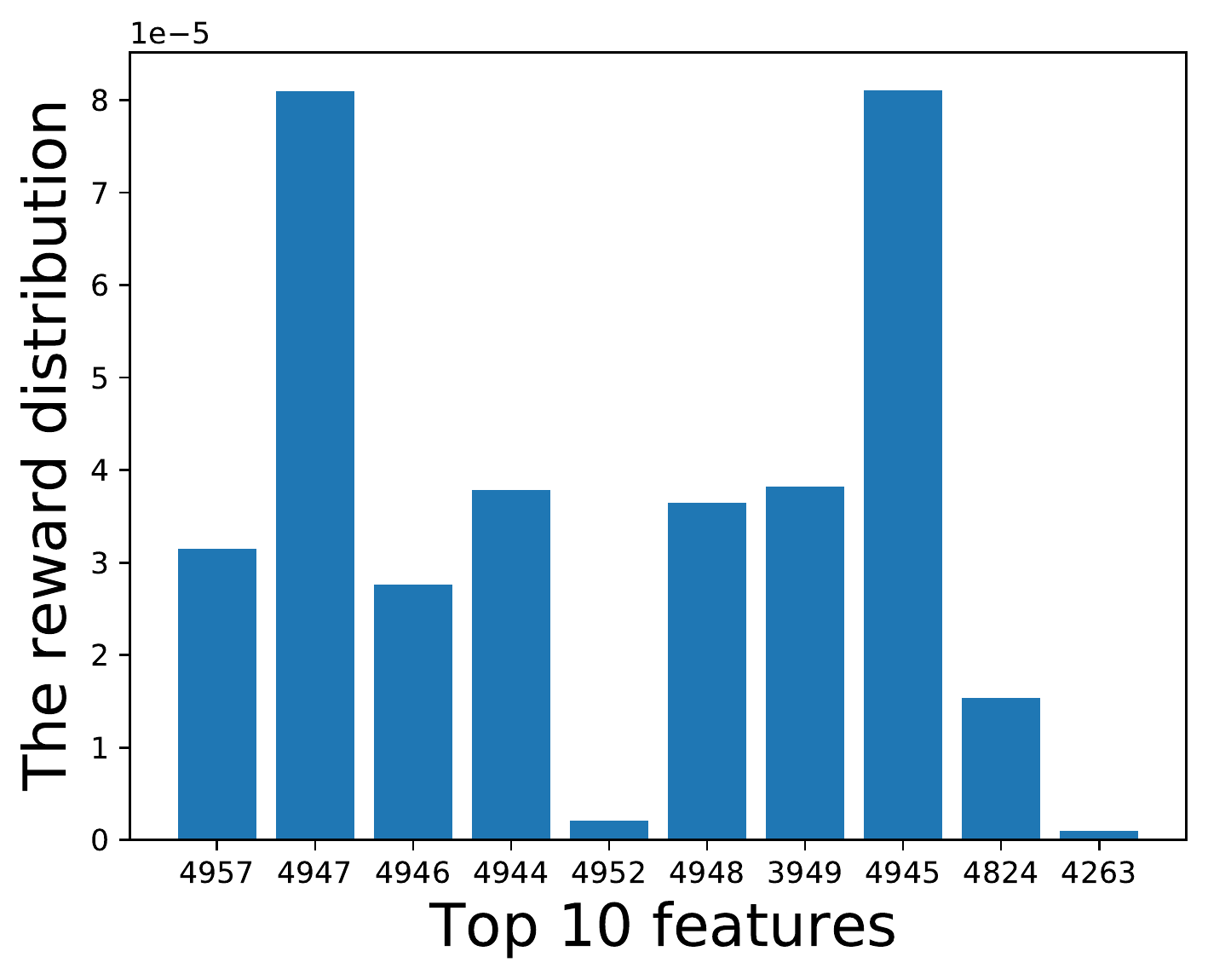}\\
{\scriptsize (a) PEDec}
\end{minipage}%
\hspace{10mm}
\begin{minipage}[t]{0.42\linewidth}
\centering
\includegraphics[width=\textwidth]{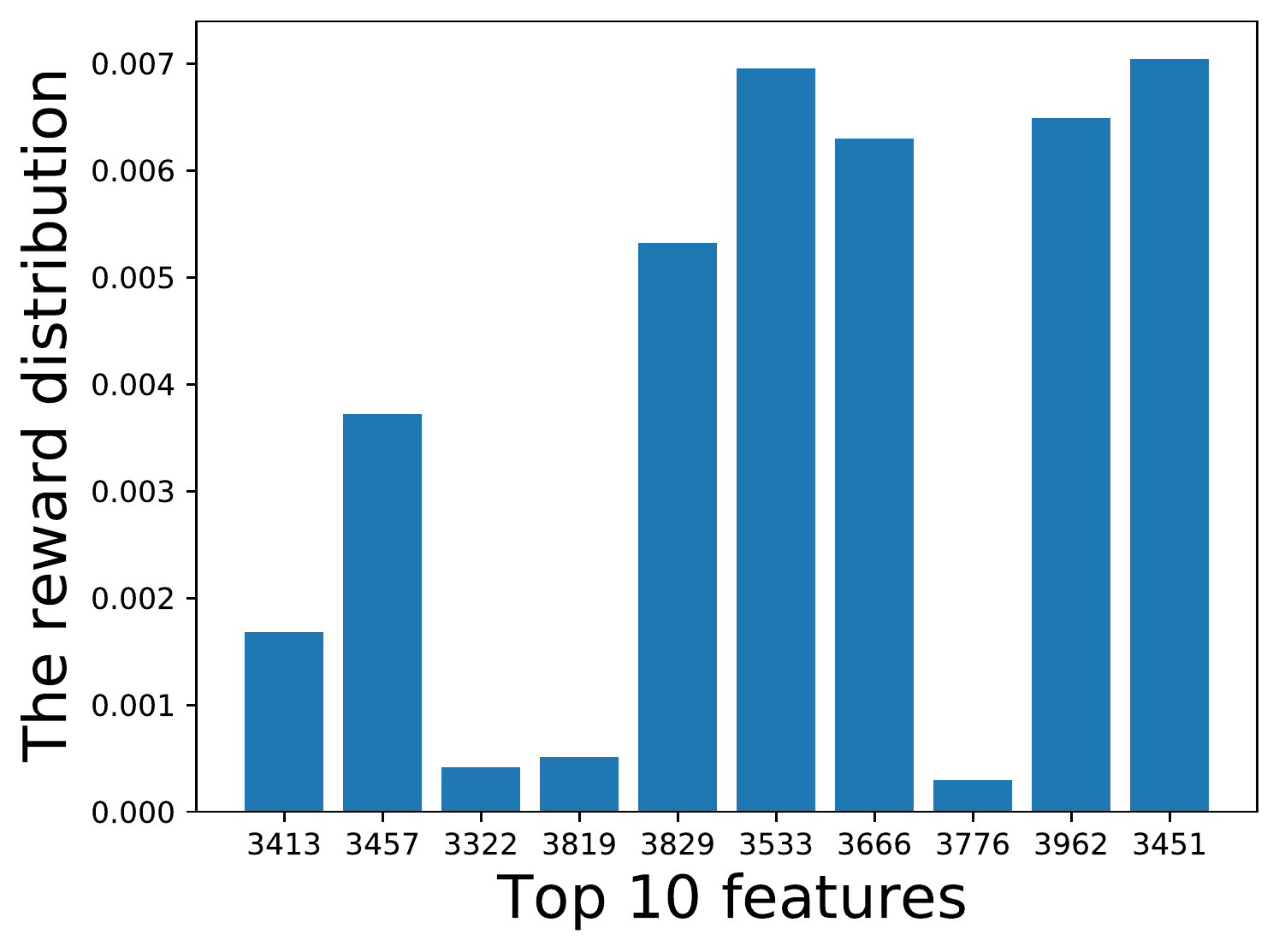}\\
{\scriptsize (b) EHR}
\end{minipage}%
\caption{The reward distribution of the top 10 sensitive features on PEDec and EHR.}
\end{figure}


The observations above denote that within the neighboring search iterations $\{t^c+T\}$. The variance of the rewards by perturbing those highly sensitive features is small. The potential explanation behind the observation can then be summarised as follows: 

\noindent According to Theorem.1 in \cite{hyhan2022iclr} (Eq.2 and 3 on Page 3 of \cite{hyhan2022iclr}), the existence of the sensitive features is one of the key origins of adversarial vulnerability of the target classifier with categorical inputs. Higher feature sensitivity level causes significantly higher instability of the classifier's decision output, if the sensitive features are perturbed. Therefore, even though the reward distribution of modifying one feature dimension may drift as more and more features are perturbed during the attack process, the impact of modifying these highly sensitive features over misleading the classifier's decision remains relatively more stable than injecting modifications over other less sensitive features. Additionally, constraining the UCB-based search within the successive few iterations help further reduce the potential variation of the reward distribution in the search process.

Based on the observations, we adopt the following two-fold design in the proposed FEAT method to handle the non-stationary reward distribution during the search process: 

\noindent \textbf{First,} we make use of the orthogonal matching pursuit (OMP) step in FEAT to select the features that can potentially bring the most variations to the classifier's decision output.  On IPS data set (1103 attributes), the 5 mostly selected features by the OMP step also appear as the top 20 sensitive features. On EHR data set (4130 attributes), the 7 most selected features by the OMP step show up in the top 20 sensitive features. The overlapping between the features useful for attack and the top ranked sensitive attributes indicates that using the OMP-guided exploration helps the attacker narrow down the combinatorial search range to focus on the most sensitive features. 

\noindent \textbf{Second,} we apply the OMP step to compute the relaxed gradient magnitudes with respect to the candidate features after every $\tau$ UCB search iterations. We then re-initialise the UCB search over the updated top $L$ features that are re-selected by the OMP operation. In this sense, we adopt the UCB-based search in the inner iterations (see Line 6 to Line 16 of Algorithm.1 in the submission) and conduct the OMP operation in the outer iterations (see Line 3 to Line 5 of Algorithm.1 in the submission). In our experimental study, $\tau =5$ and $\tau =6$ both achieve the best empirical performances. 

\textbf{On one hand}, we use the OMP operation in the outer iterations to restrict down the search range to the potentially sensitive features. The reward values with respect to those sensitive features tend to remain stable over different search rounds. \textbf{On another hand}, we reinitialize and relaunch the UCB-based search after the top $L$ features are updated in the outer iteration. Via this way, the proposed FEAT method constrains the UCB-based search within the consecutive $\tau$ inner iterations. Based on our observations, the reward distribution of the selected $L$ sensitive features can be considered empirically approximately stationary within the $\tau$ inner iterations. 

fgithub
\section{G.Accuracy of the targeted classifier over adversary-free testing data}


The testing accuracy with the targeted classifiers for different datesets is summarized in Table \ref{tab:classifier}. 

\begin{table}[t]
\centering
\caption{This is the trained model of different dateset}
\label{tab:classifier}
\scriptsize
\begin{tabular}{|c|c|c|c|c|}
\hline
\textbf{Dataset} & \textbf{Classifier} & \textbf{Accuracy} & \textbf{F1} & \textbf{AUC} \\ \hline
\textbf{Yelp-5}  & LSTM                & 0.61              & 0.60            &0.67              \\ \hline
\textbf{PEDec}   & CNN                 & 0.94              & 0.92            & 0.90             \\ \hline
\textbf{IPS}     & LSTM                & 0.92              & 0.943       & 0.987        \\ \hline
\textbf{EHR}     & LSTM                & 0.932             & 0.883       & 0.909        \\ \hline
\end{tabular}
\end{table}

In parallel to Table.\ref{tab:Yelp-5 performance} in the submission, we provide the attack performance comparison of all the involved attack algorithms with the high attack budget in Table.\ref{tab:Yelp-5 performance V2}. Compared to Table.\ref{tab:Yelp-5 performance}, a higher attack budget level allows the attack methods to introduce more modifications to the input test instances. Hence the attack success rate and the averaged number of changed features increase correspondingly for most of the attack methods in Table.\ref{tab:Yelp-5 performance V2}. However, we can find that the proposed FEAT method can consistently obtain higher \textbf{SR} and lower/similar \textbf{No.change} than the other attack baselines on Yelp-5, IPS and EHR data, and ranks the second on PEDec. Simultaneously, FEAT delivers successful attacks by costing only a small fraction of \textbf{Runtime} and \textbf{No.query}. The results confirm in a further step the merits of balancing exploration and exploitation in the attack over high-dimensional categorical data.

\begin{table*}[t]
\renewcommand{\arraystretch}{1.1}
\caption{The performances of the attack methods evaluated on \textbf{efficiency} metrics with higher attack budget than those given in Table.\ref{tab:Yelp-5 performance}: \textbf{Runtime} (average running time in seconds) and \textbf{No.query} (average no. of $f(\hat{\textbf{x}})$ evaluation), and \textbf{effectiveness} metrics:  \textbf{SR} (success rate) on four datasets. The attack time limit  $T_L$=1000 sec for these four datasets.
}
\vspace{+0.2cm}
\label{tab:Yelp-5 performance V2}
\scriptsize
\begin{subtable}[The results on Yelp-5 data]{
\resizebox{0.495\linewidth}{!}{
\begin{tabular}{lc|ccc}
\toprule
\multicolumn{2}{l|}{\textbf{Yelp-5}}                                                                                       & \multicolumn{3}{c}{\textbf{Budget = 8}}                              \\ \hline
\multicolumn{2}{l|}{\textbf{Attack Type \& Algo.}}                                                                                 & \textbf{Runtime} (sec) $\downarrow$   & \textbf{No.query} $\downarrow$    & \textbf{SR} $\uparrow$  \\ \hline
\multicolumn{1}{l|}{\multirow{2}{*}{\textbf{\begin{tabular}[c]{@{}l@{}}Domain\\ Specific\end{tabular}}}} & TextBugger       & 1.43          & 28                    & 0.75                 \\ \cline{2-5} 
\multicolumn{1}{l|}{}                                                                                    & TextFooler       & 0.19          & 155                  & 0.86                  \\ \hline
\multicolumn{1}{l|}{\multirow{2}{*}{\textbf{\begin{tabular}[c]{@{}l@{}}Black\\ Box\end{tabular}}}}       & FSGS            & 1.62          & 84000                  & 0.95                  \\ \cline{2-5} 
\multicolumn{1}{l|}{}                                                                                    & {FEAT-B} & {0.13} & {2000}  & {0.92}  \\ \hline
\multicolumn{1}{l|}{\multirow{3}{*}{\textbf{\begin{tabular}[c]{@{}l@{}}White\\ Box\end{tabular}}}}       & GradAttack      & 0.46          & 129514                 & 0.85                 \\ \cline{2-5} 
\multicolumn{1}{l|}{}                                                                                    & OMPGS           & 0.82          & 7000                   & 0.96                 \\ \cline{2-5} 
\multicolumn{1}{l|}{}                                                                                    & \textbf{FEAT}   & \textbf{0.12} & \textbf{883}   & \textbf{0.97}   \\ \bottomrule
\end{tabular}
\vspace{-1cm}
}
}\end{subtable}
\qquad
\begin{subtable}[The results on IPS data]{
\resizebox{0.495\linewidth}{!}{
\begin{tabular}{lc|ccc}
\toprule
\multicolumn{2}{l|}{\textbf{IPS}}                                                                                    & \multicolumn{3}{c}{\textbf{Budget = 6}}                                              \\ \hline
\multicolumn{2}{l|}{\textbf{Attack Type \& Algo.}}                                                                  & \textbf{Runtime} (sec) $\downarrow$   & \textbf{No.query} $\downarrow$   & \textbf{SR} $\uparrow$    \\ \hline
\multicolumn{1}{l|}{\multirow{2}{*}{\textbf{\begin{tabular}[c]{@{}l@{}}Black\\ Box\end{tabular}}}} & FSGS            & 56.2            & 70000               & 0.85                 \\ \cline{2-5} 
\multicolumn{1}{l|}{}                                                                              & {FEAT-B} & {19.5} & {2800}  & {0.75}  \\ \hline
\multicolumn{1}{l|}{\multirow{3}{*}{\textbf{\begin{tabular}[c]{@{}l@{}}White\\ Box\end{tabular}}}} & GradAttack      & 23.7          & 2500          & 0.75                 \\ \cline{2-5} 
\multicolumn{1}{l|}{}                                                                              & OMPGS           & 0.99           & 92.1          & 0.83                 \\ \cline{2-5} 
\multicolumn{1}{l|}{}                                                                              & \textbf{FEAT}   & \textbf{0.28}  & \textbf{121} & \textbf{0.94}  \\ \bottomrule
\end{tabular}
}
\vspace{-1cm}
}\end{subtable}
\qquad

\end{table*}

\end{document}